%% file: main.tex
\documentclass[11pt]{article}

\usepackage[final]{acl}

\usepackage{times}
\usepackage{latexsym}
\usepackage[T1]{fontenc}
\usepackage[utf8]{inputenc}
\usepackage{microtype}
\usepackage{inconsolata}
\usepackage{placeins}

\usepackage{amsmath}
\usepackage{amssymb}
\usepackage{amsthm}
\usepackage{bm}

\newtheorem{theorem}{Theorem}

\theoremstyle{definition}
\newtheorem{definition}{Definition}

\theoremstyle{remark}
\newtheorem{remark}{Remark}

\usepackage{graphicx}
\usepackage{booktabs}
\usepackage{multirow}
\usepackage{adjustbox}
\usepackage{enumitem}
\usepackage{subcaption}
\usepackage{tikz}
\usetikzlibrary{
  arrows.meta,
  positioning,
  shapes.geometric
}
\usetikzlibrary{calc}
\usepackage{float}

\newcommand{\shownotes}{1} 
\ifnum \shownotes = 1
	\newcommand{\authnote}[2]{{$\ll$\textsf{\footnotesize #1 notes: #2}$\gg$}}
\else
	\newcommand{\authnote}[2]{}
\fi

\title{Learning Dynamic Representations and Policies from Multimodal Clinical Time-Series with Informative Missingness}

\author{
Zihan Liang\thanks{Equal contribution. Code and reproducibility materials are available at \url{https://github.com/CausalMLResearch/OPL-MT-MNAR} under the MIT License.}
\quad
Ziwen Pan\footnotemark[1]
\quad
Ruoxuan Xiong \\
Emory University, Atlanta, USA \\
\texttt{\{zihan.liang, ziwen.pan, ruoxuan.xiong\}@emory.edu}
}

\begin{document}
\maketitle

\begin{abstract}
Multimodal clinical records contain structured measurements and clinical notes recorded over time, offering rich temporal information about the evolution of patient health. Yet these observations are sparse, and whether they are recorded depends on the patient’s latent condition. Observation patterns also differ across modalities, as structured measurements and clinical notes arise under distinct recording processes. While prior work has developed methods that accommodate missingness in clinical time series, how to extract and use the information carried by the observation process itself remains underexplored. We therefore propose a patient representation learning framework for multimodal clinical time series that explicitly leverages informative missingness. The framework combines (1) a multimodal encoder that captures signals from structured and textual data together with their observation patterns, (2) a Bayesian filtering module that updates a latent patient state over time from observed multimodal signals, and (3) downstream modules for offline treatment policy learning and patient outcome prediction based on the learned patient state. We evaluate the framework on ICU sepsis cohorts from MIMIC-III, MIMIC-IV, and eICU. It improves both offline treatment policy learning and adverse outcome prediction, achieving FQE 0.679 versus 0.528 for clinician behavior and AUROC 0.886 for post-72-hour mortality prediction on MIMIC-III.
\end{abstract}

\section{Introduction}
\label{sec:intro}

Electronic health records are multimodal, consisting of structured data, such as vital signs and laboratory measurements, as well as clinical texts, such as notes and reports. These data are recorded longitudinally and encode rich temporal information about how patient health evolves over time. This makes it possible to learn dynamic patient representations that support both outcome prediction and sequential clinical decision-making. However, two key features exist in clinical observations that complicate how such representations should be learned.

\begin{figure}[t!]
\centering
\includegraphics[width=\columnwidth]{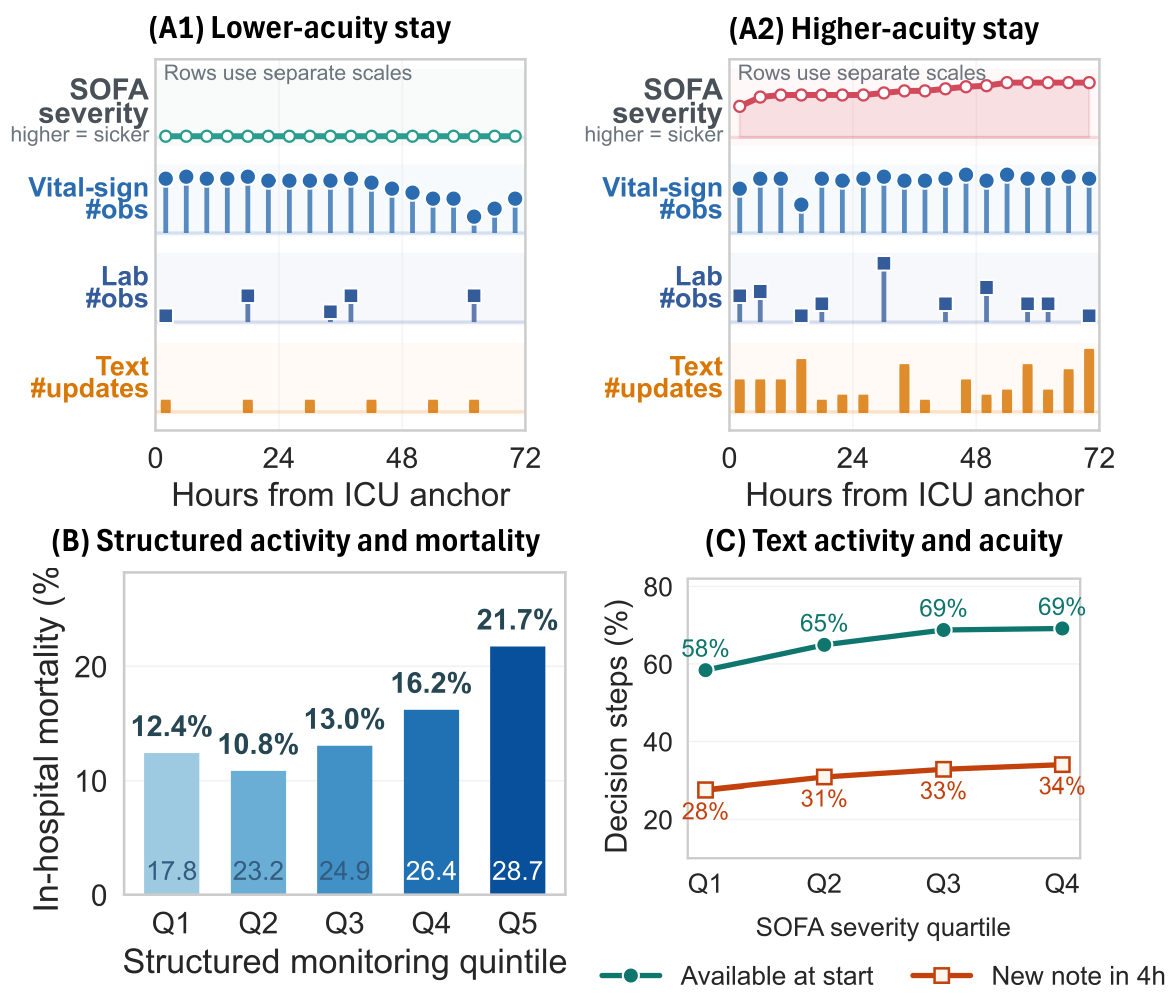}
\caption{\textbf{Text and structured observations exhibit temporal MNAR patterns in ICU care.}
\textbf{(A)} Two representative MIMIC-III ICU stays illustrate how acuity and observation processes co-evolve. The top row shows SOFA severity (higher = sicker); the lower rows show counts of vital-sign observations, laboratory observations, and text updates within each 4-hour bin. Compared with a lower-acuity stay (A1), a higher-acuity stay (A2) shows rising SOFA, denser laboratory measurements, and burstier documentation.
\textbf{(B)} Across ICU stays, greater structured monitoring intensity is associated with higher in-hospital mortality.
\textbf{(C)} Clinical text follows the same endogenous pattern: higher-acuity decision steps are more likely to already have text available and to receive new text within the next 4 hours.}
\label{fig:intro}
\end{figure}

First, they are often sparse and irregular, and their observation process depends on both clinician decisions and the patient’s underlying health state. Figure~\ref{fig:intro}(A) illustrates this pattern using two representative ICU trajectories from MIMIC-III~\citep{johnson2016mimic}: compared with a lower-acuity trajectory, a higher-acuity trajectory exhibits denser laboratory monitoring and more frequent text updates. The same pattern also appears at the cohort level. Using MIMIC-IV~\citep{johnson2023mimic}, Figure~\ref{fig:intro}(B) shows that greater structured monitoring intensity is associated with higher in-hospital mortality. Figure~\ref{fig:intro}(C) shows a parallel relationship for text: higher-acuity patients are more likely to have text available at a given decision step and to receive new text in the following step.

Second, observation patterns differ systematically across modalities because different types of clinical data are generated through different mechanisms. Vital signs are often collected more routinely, laboratory tests need to be ordered, and text updates depend even more directly on clinician documentation behavior. This contrast is visible in Figure~\ref{fig:intro}(A): even within the same patient trajectory, the temporal availability of structured measurements and text updates evolves differently.

Taken together, these patterns suggest that observation processes are informative about patient state but should be used carefully, as their meanings differ across modalities. For structured clinical time series, prior work has proposed methods that incorporate informative missingness through masks and time gaps~\citep{che2018recurrent}. In the multimodal setting, \citet{liang2025causal} also explicitly models informative missingness, but without temporal dynamics. However, it remains open how to leverage informative missingness over time across modalities to learn patient representations.

We propose OPL-MT-MNAR (Off-Policy Learning under Multimodal Observations with Temporal Missing-Not-At-Random Patterns), a framework that explicitly leverages informative missingness in multimodal clinical records. Our approach has two stages. In the first stage, we learn a dynamic latent representation from the multimodal observations available up to the current time. In the second stage, we use the learned patient representation for outcome prediction and offline policy learning.

The first stage is motivated by \textit{Bayesian filtering} and consists of two components. The first is a multimodal encoder that learns a unified representation from the data observed so far. For structured data, we construct the embedding using an extension of Gated Recurrent Units-Decay (GRU-D)~\citep{che2018recurrent} together with additional missingness-aware features (time since last observation, cumulative observation counts, missing rates, and windowed observation frequency). For clinical text, we introduce a temporal documentation factor that is updated at each time step and summarizes the observation pattern of the text observed so far. This factor is then used both to refine the text embedding and to guide the fusion of text and structured-data embeddings into a unified representation.

The second component models \textit{patient dynamics} through a latent belief state learned via variational inference. This belief state captures underlying health dynamics together with the cumulative effects of past treatment actions. It is then combined with the unified representation from the multimodal encoder to form a posterior patient state for downstream tasks. We \textit{show theoretically} that without such a belief state, the multimodal encoder alone may fail to preserve sufficient information about treatment history for sequential decision-making.

In the second stage, we use the posterior patient state for multiple downstream tasks. One task is offline treatment policy optimization using expectile regression~\citep{kostrikov2022offline}, which accommodates delayed rewards. The other is outcome prediction. Jointly learning these tasks allows the shared patient representation to benefit from positive transfer across tasks.

We evaluate our framework on MIMIC-III, MIMIC-IV, and eICU~\citep{pollard2018eicu}, which together cover complementary text-observation regimes and cross-institutional generalization. Our empirical focus is ICU sepsis care with 4-hour decision steps over the first 72 ICU hours. Across these settings, the learned state supports both clinically meaningful prediction and offline treatment optimization, with the largest gains appearing in high-acuity settings where observation processes are most informative. Concretely, the OPL-MT-MNAR policy achieves FQE 0.679 on MIMIC-III, 0.634 on MIMIC-IV, and 0.604 on eICU, compared with clinician behavior at 0.528, 0.521, and 0.534, respectively. Using the same learned representation, we also achieve AUROC 0.886 for post-72-hour mortality prediction on MIMIC-III, with the clearest policy improvements appearing in the highest-acuity subgroup.

\section{Problem Formulation}
\label{sec:problem}
We consider a finite-horizon partially observable Markov decision process (POMDP) $\mathcal{M} = (\mathcal{S}, \mathcal{A}, \mathcal{O}, P, \Omega, R, \gamma, H)$~\citep{kaelbling1998planning,hauskrecht2000planning}, where $\mathcal{S}$ is the latent state space, $\mathcal{A}$ is the discrete action space, $\mathcal{O}$ is the observation space, $P: \mathcal{S} \times \mathcal{A} \to \Delta(\mathcal{S})$ is the transition kernel, $\Omega: \mathcal{S} \to \Delta(\mathcal{O})$ is the observation function, $R$ is the reward function, $\gamma \in [0,1)$ is the discount factor, and $H$ is the horizon. The true patient state $s_h \in \mathcal{S}$ at decision step $h$ is not directly observed.

At decision step $h$, the agent \emph{partially observes} the following information
\[
o_h = (\bm{y}_h^{\mathsf{s}}, \bm{m}_h^{\mathsf{s}}, \bm{y}_h^{\mathsf{t}}, \bm{m}_h^{\mathsf{t}}) \in \mathcal{O}.
\]
For structured data, $\bm{y}_h^{\mathsf{s}} \in \mathbb{R}^{|\mathcal{T}_h| \times D}$ denotes the measurement values recorded at observation times $\mathcal{T}_h$ within decision step $h$, and $\bm{m}_h^{\mathsf{s}} \in \{0,1\}^{|\mathcal{T}_h| \times D}$ indicates whether each entry is observed. For text data, $\bm{y}_h^{\mathsf{t}} = \{y_{h}^{\mathsf{t},j}\}_{j \in \mathcal{M}^{\mathsf{t}}}$ denotes the collection of raw text observations across modalities, and $\bm{m}_h^{\mathsf{t}} = [m_{h}^{\mathsf{t},j}]_{j \in \mathcal{M}^{\mathsf{t}}}$ is the corresponding binary indicator vector, where $m_{h}^{\mathsf{t},j}$ records whether text modality $j$ is observed at step $h$. These raw text observations are encoded into a step-level text embedding $e_h^{\mathsf{t}} \in \mathbb{R}^{d_e}$ before entering the multimodal fusion module.

In addition, static patient features $x$ (e.g., age and gender) are available throughout the trajectory. We let $I_h = \{x, o_1, \cdots, o_h\}$ be the information set accumulated up to step $h$, comprising static features together with structured and textual information collected dynamically.

Episodes may terminate at step $h^* \leq H$ if the patient dies, is discharged, or reaches the end of the observation window. The logged episode ends at $h^*$, so no observations, actions, or rewards are defined for $h > h^*$. Rewards are sparse within the realized episode: $r_{h^*} = +1$ for survival and $-1$ for in-hospital mortality, while $r_h = 0$ for $h < h^*$. 

To quantify observation patterns, we additionally define modality-specific summary statistics. For structured data, let $\bm{\delta}_h^{\mathsf{s}} \in \mathbb{R}_+^{|\mathcal{T}_h| \times D}$ denote the time since last observation for each variable at each timestamp within decision step $h$. For text data, let $\bm{n}_h^{\mathsf{t}} = [n_{h}^{\mathsf{t},j}] \in \mathbb{Z}_+^{|\mathcal{M}^{\mathsf{t}}|}$ denote the number of text observations in modality $j$ at step $h$; for example, a single step may contain multiple nursing notes. We further define the documentation density
\[ \kappa_h^{\mathsf{t}} = \frac{1}{K}\sum_{u=h-K+1}^{h} \sum_{j \in \mathcal{M}^{\mathsf{t}}} n_{u}^{\mathsf{t},j},\]
which summarizes recent documentation activity over a rolling window of length $K$. Together, these summaries capture behavior-driven observation timing, documentation burstiness, and \textit{missing-not-at-random} (MNAR) patterns that may correlate with patient severity.  Appendix~\ref{app:notation} summarizes the notation used throughout the paper.

\paragraph{Learning Objectives.}
Given the static dataset $\mathcal{D} = \{(o_h^{(i)}, a_h^{(i)}, r_h^{(i)})\}_{i,h}$ collected under behavior policy $\pi_\beta$ (clinician decisions), we aim to achieve three interconnected objectives:

\textit{Q1 (State Learning).} Learn an encoder $g_\theta$ such that the state $s_h = g_\theta(I_h)$ captures sufficient information from multimodal observations with MNAR patterns and behavior-driven text observations. We verify state quality through reconstruction: a decoder $f_\phi$ should recover the step-level observations $(\bm{y}_h^{\mathsf{s}}, \bm{m}_h^{\mathsf{s}})$, ensuring that MNAR information is preserved in the learned representation.

\textit{Q2 (Policy Optimization).} Under the offline constraint, learn a policy $\pi(a_h \mid s_h)$ that maximizes expected return while avoiding catastrophic errors from out-of-distribution actions.

\textit{Q3 (Outcome Prediction).} Predict clinical outcomes (e.g., post-72-hour mortality) from the terminal state representation. This grounds learned states in clinically meaningful signals.

\section{Method}
\label{sec:method}

We propose a two-stage framework as shown in Figure~\ref{fig:architecture}. Stage 1 (Section~\ref{sec:stage1}) learns patient health state representations from multimodal observations with structured-measurement MNAR and text MNAR; Stage 2 (Section~\ref{sec:stage2}) uses these states to optimize treatment policies via Implicit Q-Learning and to predict outcomes. 

\begin{figure*}[t]
\centering
\includegraphics[width=\textwidth]{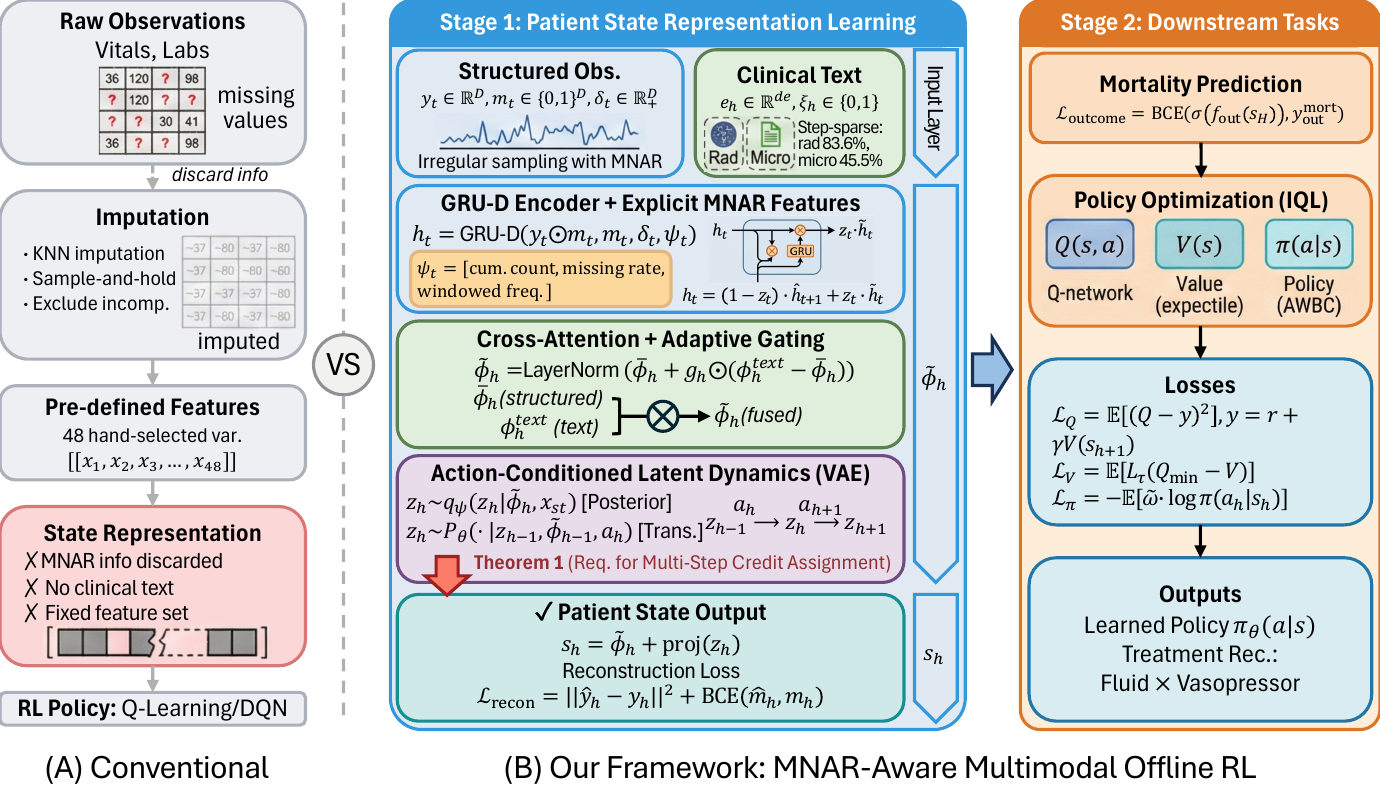}
\caption{OPL-MT-MNAR: \textbf{Stage 1} learns state $s_h$ with structured-measurement MNAR, documentation-process MNAR, and action-conditioned latent dynamics; \textbf{Stage 2} uses $s_h$ for outcome prediction and policy optimization.}
\label{fig:architecture}
\end{figure*}

\subsection{Patient State Representation Learning}
\label{sec:stage1}

To learn patient state, we adopt a \textit{Bayesian filtering} perspective with the causal diagram shown in Figure \ref{fig:causal}. At each decision step $h$, we first encode the multimodal observations into a unified representation $\phi_h$ that explicitly captures both structured-measurement MNAR and text MNAR (Section~\ref{sec:obs_embedding}). We then maintain a \textit{latent belief state} $z_h$, learned via \textit{variational inference}, whose transition to $z_{h+1}$ is conditioned on the treatment action (Section~\ref{sec:latent_dynamics}). Finally, we construct the \textit{posterior patient health state representation} $s_h$ by combining the current observation representation $\phi_h$ with the latent belief state $z_h$ (Section~\ref{sec:patient_state}). The resulting $s_h$ is used for downstream tasks.

\subsubsection{MNAR-Aware Observation Embedding}
\label{sec:obs_embedding}

We construct a unified representation $\phi_h$ through two steps: first encoding irregular structured observations with explicit MNAR modeling to obtain $\bar{\phi}_h^{\mathsf{s}}$, then incorporating clinical text together with its documentation process to produce $\phi_h$. 

\paragraph{Structured Observation Encoding.}

For structured observations, we use an encoder built on GRU-D~\citep{che2018recurrent}. The key idea is that if a variable has been missing for a long time, then its last observed value becomes less reliable, and the influence of stale information should gradually decay over time. To capture this effect, we introduce learned hidden-state and input decay factors at timestamp $u$ as functions of the time-since-last-observation vector:
\begin{equation*}
\begin{aligned}
\xi_{\phi,u} &= \exp\big({-}\max(0, W_{\phi,\xi} \cdot \mathrm{mean}(\delta_u) + b_{\phi,\xi})\big), \\
\bm{\xi}_{y,u} &= \exp\big({-}\max(0, W_{y,\xi} \delta_u + b_{y,\xi})\big),
\end{aligned}
\end{equation*}
where $\xi_{\phi,u} \in (0,1]$ controls hidden-state decay, $\bm{\xi}_{y,u} \in (0,1]^D$ controls input decay, $\xi_{y,u}^d$ denotes the $d$-th entry of $\bm{\xi}_{y,u}$, and all weights and biases are trainable parameters. 

Let $y_u^{\mathsf{s},d}$ denote the value of the $d$-th structured variable at time $u$, and let $m_u^{\mathsf{s},d}\in\{0,1\}$ indicate whether that value is observed. When a measurement is missing, we decay its last observed value toward a default value given by the empirical mean
\begin{align*}
    \hat{y}_u^{\mathsf{s},d} =& m_u^{\mathsf{s},d} \cdot y_u^{\mathsf{s},d} + (1 - m_u^{\mathsf{s},d}) \cdot \\
    & \big(\xi_{y,u}^d \cdot y_{u'}^{\mathsf{s},d} + (1-\xi_{y,u}^d) \cdot \mu^d \big),
\end{align*}
where $y_{u'}^{\mathsf{s},d}$ is the most recent observed value of the $d$-th variable before time $u$, and $\mu^d$ is the empirical mean of that variable.

Let $\phi_{u-1}^{\mathsf{s}}$ denote the hidden state from the previous timestamp. We first apply hidden-state decay, $\hat{\phi}^{\mathsf{s}}_{u-1} = \xi_{\phi,u} \odot \phi^{\mathsf{s}}_{u-1}$,
and then update the hidden state using GRU-D-style gates. In addition to the decayed inputs, we incorporate explicit MNAR features $\bm{\psi}_u^{\mathsf{s}}$ that summarize monitoring patterns predictive of patient severity, including cumulative observation counts, missing rates, and windowed observation frequencies. The gated updates are
\begin{align*}
    r_u =& \sigma(W_{r,y} \hat{y}_u^{\mathsf{s}} + W_{r,\phi} \hat{\phi}^{\mathsf{s}}_{u-1} + W_{r,\psi} \bm{\psi}_u^{\mathsf{s}} + b_r ) \\
    \eta_u =& \sigma(W_{\eta,y} \hat{y}_u^{\mathsf{s}} + W_{\eta,\phi} \hat{\phi}^{\mathsf{s}}_{u-1} + W_{\eta,\psi} \bm{\psi}_u^{\mathsf{s}} + b_\eta ) \\
    \tilde{\phi}^{\mathsf{s}}_u =&\mathrm{tanh}( W_{y} \hat{y}_u^{\mathsf{s}} + W_\phi (r_u \odot \hat{\phi}^{\mathsf{s}}_{u-1}) + W_\psi \bm{\psi}_u^{\mathsf{s}} + b) \\
    \phi_u^{\mathsf{s}} =& (1 - \eta_u) \odot \hat{\phi}^{\mathsf{s}}_{u-1} + \eta_u \odot \tilde{\phi}^{\mathsf{s}}_u \,,
\end{align*}
where $\sigma$ denotes the sigmoid function
and all $W$ matrices and $b$ vectors are trainable parameters. For decision step $h$, we take the hidden state at the last timestamp within the step as the structured embedding $\bar{\phi}_h^{\mathsf{s}}$. See Appendix~\ref{app:encoder} for more details. 

\paragraph{Sparse Text Fusion.}

Clinical text is highly informative but is observed irregularly, and its availability is itself shaped by clinician documentation behavior. We therefore begin by introducing a \textit{documentation-process factor} to summarize this behavior over time. At each decision step $h$, this factor is updated from note presence, text recency, and recent documentation density:
\begin{equation*}
\begin{aligned}
\eta_h^{\mathsf{t}} &= \operatorname{MLP}\big(\bm{m}_h^{\mathsf{t}}, \delta_h^{\mathsf{t}}, \kappa_h^{\mathsf{t}}\big) \\
F_h^{\text{doc}} &= \operatorname{GRU}\big(F_{h-1}^{\text{doc}}, \eta_h^{\mathsf{t}}\big),
\end{aligned}
\end{equation*}
where $\eta_h^{\mathsf{t}}$ is a step-level summary of the text observation pattern, and $F_h^{\text{doc}}$ accumulates these signals over time. Importantly, this documentation-process factor is constructed only from the observation process and does not directly use text content.

Next, we construct a representation of the available text at decision step $h$. To align textual information with the structured representation $\bar{\phi}_h^{\mathsf{s}}$, we obtain the representation $\phi_h^{\mathsf{t}}$ by applying multi-head cross-attention from $\bar{\phi}_h^{\mathsf{s}}$ to embedding $e_h^{\mathsf{t}}$:
\begin{equation*}
\begin{aligned}
\phi_h^{\mathsf{t}} &= \operatorname{MultiHead}\Big(W_Q \bar{\phi}_h^{\mathsf{s}},\,W_K e_h^{\mathsf{t}},\,W_V e_h^{\mathsf{t}}\Big).
\end{aligned}
\end{equation*}
Here, $W_Q$, $W_K$, and $W_V$ are trainable projection matrices. When a text modality is unavailable, we use a learned missing embedding rather than dropping that modality entirely.

Finally, we adaptively fuse the structured representation $\bar{\phi}_h^{\mathsf{s}}$ and the text representation $\phi_h^{\mathsf{t}}$ using the documentation-process factor $F_h^{\text{doc}}$. This allows the model to weigh text based on its semantic content and how it is documented. Specifically, we use the following gating mechanism:
\begin{equation*}
\begin{aligned}
\hat{\phi}_h^{\mathsf{t}} &= \phi_h^{\mathsf{t}} + W_d F_h^{\text{doc}}, \\
g_h &= \sigma\big(W_g [\bar{\phi}_h^{\mathsf{s}}; \hat{\phi}_h^{\mathsf{t}}; F_h^{\text{doc}}] + b_g\big), \\
\phi_h &= \text{LayerNorm}\big(\bar{\phi}_h^{\mathsf{s}} + g_h \odot (\hat{\phi}_h^{\mathsf{t}} - \bar{\phi}_h^{\mathsf{s}})\big).
\end{aligned}
\end{equation*}
Here, $\hat{\phi}_h^{\mathsf{t}}$ augments the text representation with documentation-process information. The gate $g_h$ adaptively controls how much the \emph{unified representation} $\phi_h$ should move from the structured embedding toward the text-enhanced representation. In this way, the model can distinguish between settings such as no text, stale text, and a burst of newly updated text, even when the underlying text content is similar. See Appendix~\ref{app:text_fusion} for more details.

\subsubsection{Latent Belief State via Variational Inference}
\label{sec:latent_dynamics}
The unified representation $\phi_h$ summarizes the information available at the current decision step. However, it may not fully capture the underlying health dynamics or the cumulative effects of past treatment actions. To address this limitation, we use Bayesian filtering and introduce a latent belief state $z_h$ with action-conditioned dynamics. See Appendix~\ref{app:partial_obs} for illustrative healthcare examples.

\input{diagram}

By conditioning on action $a_h$, the transition from $z_h$ to $z_{h+1}$ captures how intervention histories shape patient trajectories and treatment responsiveness. In contrast, if $z_{h+1}$ depends only on the previous belief state $z_h$ and the current representation $\phi_{h+1}$, but not on $a_h$, then the resulting model may fail to support policies that optimize long-term rewards. The key reason is that, under the causal structure in Figure~\ref{fig:causal}, $\phi_{h+1}$ is a deterministic function of the recorded observations, which implies $\partial \phi_{h'} / \partial a_h = 0$ for all future steps $h' > h$.

\begin{definition}[Action-Independent Dynamics]
\label{def:action_independent}
In the offline RL setting where $\partial \phi_{h'} / \partial a_h = 0$ (as established above), a latent dynamical system has \emph{action-independent dynamics} if the transition function satisfies $z_{h+1} = f(z_h, \phi_h, \omega_h)$ with $\omega_h \perp\!\!\!\perp a_h$, where $\omega_h$ is exogenous noise independent of the learned policy's action.
\end{definition}

\begin{theorem}[Necessity of Action-Conditioning]
\label{thm:controllability}
Let the policy objective be $J(\pi) = \mathbb{E}_{\tau \sim \pi}[\sum_{h=0}^{H-1} \xi^h r_h]$ where $r_h = R(s_h, a_h)$ and $s_h = g_\theta(\phi_h, z_h)$ for a differentiable combination function $g_\theta$. Under action-independent dynamics (Definition~\ref{def:action_independent}), the policy gradient satisfies:
$$\frac{\partial}{\partial \pi} \mathbb{E}\left[\sum_{h'=h+1}^{H-1} \gamma^{h'} r_{h'} \;\Big|\; s_h, a_h\right] = 0 \quad \text{for all } h.$$
That is, current actions have no gradient signal from future rewards.
\end{theorem}

With terminal-only rewards, this implies the policy gradient is zero for all non-terminal steps, making it impossible to learn that early interventions affect long-term outcomes. The proof is in Appendix~\ref{app:theory}. Importantly, this theorem does not compete with MNAR-aware observation modeling: richer observation encoding improves what the model \emph{sees} at step $h$, while action-conditioned latent dynamics preserve what the policy can still \emph{learn} from future rewards.

\paragraph{VAE Formulation.}
We then parameterize $z_h$ using a variational autoencoder (VAE):
\begin{equation}
\label{eq:dynamics}
\begin{aligned}
z_{h+1} &\sim p_\theta(z_{h+1} \mid z_h, \phi_h, a_h) \\
&= \mathcal{N}\big(\mu_\theta(z_h, \phi_h, a_h), \sigma_\theta^2(z_h, \phi_h, a_h)\big)
\end{aligned}
\end{equation}
where $\mu_\theta$ and $\sigma_\theta$ are parameterized by neural networks, and $z_0 \sim \mathcal{N}(0, I)$. The posterior $q_\psi(z_{h+1} \mid \phi_{h+1}, x)$ incorporates the next-step observations during training. A dynamics loss $\mathcal{L}_{\text{dyn}}$ enforces consistency between predicted and inferred states. Appendix~\ref{app:dynamics} provides the full VAE architecture, dynamics loss, and KL regularization details.

\subsubsection{Posterior Patient State Representation}
\label{sec:patient_state}

The final patient health state $s_h$ combines unified observation representation $\phi_h$ with the latent belief $z_h$ via a learnable combination function:
\begin{equation*}
s_h = g_\theta(\phi_h, z_h),
\end{equation*}
for a parametrized function $g_\theta$. In our implementation, we use a residual additive form $g_\theta(\phi_h, z_h) = \phi_h + \text{proj}(z_h)$, where $\text{proj}(\cdot)$ is a linear projection. This choice preserves the representation $\phi_h$ while augmenting it with belief $z_h$ that contains action-related information. Because $z_h$ is latent and stochastic under Eq.~\eqref{eq:dynamics}, the induced state $s_h$ is also stochastic conditional on the observation history $I_h$; throughout the paper, losses involving $s_h$ are therefore understood as expectations over the corresponding latent-state draws.

\paragraph{State Verification via Reconstruction.}
To verify that the learned state captures sufficient information (Q1), we use a reconstruction objective during representation pre-training: 
\begin{equation*}
\begin{aligned}
\mathcal{L}_{\text{recon}} &= \lambda_{\text{obs}} \|\tilde{\bm{y}}_h^{\mathsf{s}} - \bm{y}_h^{\mathsf{s}}\|^2 + \lambda_{\text{mask}} \operatorname{BCE}(\tilde{\bm{m}}_h^{\mathsf{s}}, \bm{m}_h^{\mathsf{s}}) \\
&\quad + \lambda_{\text{text}} \Big(\|\tilde{e}_h^{\mathsf{t}} - e_h^{\mathsf{t}}\|^2 + \|\tilde{\eta}_h^{\mathsf{t}} - \eta_h^{\mathsf{t}}\|^2\Big),
\end{aligned}
\end{equation*}
where $\tilde{\bm{y}}_h^{\mathsf{s}}$, $\tilde{\bm{m}}_h^{\mathsf{s}}$, $\tilde{e}_h^{\mathsf{t}}$, and $\tilde{\eta}_h^{\mathsf{t}}$ are reconstructed from $s_h$ using MLP-based decoders. The reconstruction terms for $\tilde{\bm{y}}_h^{\mathsf{s}}$ and $\tilde{e}_h^{\mathsf{t}}$ encourage the state to preserve sufficient information about the structured observations and encoded text content. The term for $\tilde{\bm{m}}_h^{\mathsf{s}}$ encourages the state to retain structured-data MNAR patterns, while the term for $\tilde{\eta}_h^{\mathsf{t}}$ ensures that the text documentation process is also captured.

\subsection{Policy Learning and Outcome Prediction}
\label{sec:stage2}

We next use the learned state representations from Stage 1 for two downstream tasks: treatment policy optimization and adverse outcome prediction.

\subsubsection{Offline Policy Optimization}
\label{sec:policy}

We adopt Implicit Q-Learning (IQL)~\citep{kostrikov2022offline} for policy optimization. We first learn value functions from offline data without querying out-of-distribution actions. Next, we extract a policy via advantage-weighted behavioral cloning.

\paragraph{Value Function Learning.}
To mitigate overestimation bias, we use double Q-learning~\citep{hasselt2010double} and maintain two Q-networks:
\begin{equation*}
\begin{aligned}
\mathcal{L}_{Q_j} &= \mathbb{E}_{(s_h,a_h,r_h,s_{h+1},d_h)\sim\mathcal{D}}\Big[\big(Q_j(s_h,a_h) - y_h^{\text{tgt}}\big)^2\Big]
\end{aligned}
\end{equation*}
for $j \in \{1,2\}$, where both critics are trained using the same bootstrap target
\[y_h^{\text{tgt}} = r_h + \gamma (1-d_h) V_{\text{tgt}}(s_{h+1}),\]
$d_h \in \{0,1\}$ indicates whether the episode terminates at step $h$, and $V_{\text{tgt}}$ is a slowly updated target copy of the value function $V$ used to stabilize bootstrapping (Appendix~\ref{app:iql_details}). The expectation is taken over one-step transitions in the offline dataset, so the summation over decision steps is implicit in the dataset average. 
Value function $V$ is trained via expectile regression~\citep{kostrikov2022offline}:
\begin{equation*}
\mathcal{L}_V = \mathbb{E}_{(s_h, a_h) \sim \mathcal{D}} \Big[\big|\tau - \mathbb{I}(A_h < 0)\big| \cdot A_h^2\Big],
\end{equation*}
where $A_h = \min(Q_1(s_h,a_h), Q_2(s_h,a_h)) - V(s_h)$ is the advantage and $\tau \in (0.5, 1)$ is the expectile parameter. This asymmetric loss pushes $V$ toward higher Q-values while grounded in the data distribution.
\paragraph{Policy Extraction.}
The policy is extracted via advantage-weighted behavioral cloning:
\begin{equation*}
\begin{aligned}
\mathcal{L}_{\pi} &= -\mathbb{E}_{(s_h,a_h)\sim\mathcal{D}}\bigl[
\min\bigl(\exp(A_h/\beta), w_{\max}\bigr)\, \\
&\qquad\qquad\qquad\qquad \log \pi(a_h \mid s_h)\bigr],
\end{aligned}
\end{equation*}
where $\beta > 0$ controls deviation from clinician behavior and $w_{\max}$ prevents large weights. This formulation stays close to the behavior policy while improving on it. This is essential when a distribution shift leads to harmful recommendations.

\paragraph{Action Selection.} We first draw the latent belief $z_h$ from its predictive distribution, obtain the patient state $s_h = g_\theta(\phi_h, z_h)$, and then sample an action from $\pi(\cdot \mid s_h)$. Equivalently, the induced action distribution conditional on the available history is the marginal
\[\pi(a_h \mid I_h) = \int \pi\!\left(a_h \mid g_\theta(\phi_h, z_h)\right) p(z_h \mid I_h)\,dz_h,\]
which makes the uncertainty in $z_h$ explicit while keeping the notation in the policy and value losses compact. See more details in Appendix~\ref{app:iql_details}. 

\subsubsection{Outcome Prediction}
\label{sec:outcome}
Among patients who remain alive through the 72-hour observation window ($H_i = H$ and $r_H = +1$), we predict subsequent clinical outcomes $y_{\mathrm{out}}^{(1)}, \ldots, y_{\mathrm{out}}^{(K)}$. Our primary outcome is post-72-hour in-hospital mortality. This auxiliary target is clinically meaningful while remaining distinct from the RL reward. As a result, early-terminated episodes contribute to representation learning and policy optimization, but not to this auxiliary prediction task. We define the multitask outcome prediction loss as
$$\mathcal{L}_{\text{out}} = \sum_{k=1}^{K} \lambda_k \, \ell_k\!\left(f_{\text{out}}^{(k)}(s_H), y_{\mathrm{out}}^{(k)}\right),$$
where $f_{\text{out}}^{(k)}$ denotes the prediction head for outcome $k$, $\ell_k$ is an appropriate loss function (e.g., binary cross-entropy for classification), and $\lambda_k$ controls the relative weight of each outcome. 

\begin{remark}[Information transfer across tasks]
Empirically, we observe that jointly learning policy optimization and outcome prediction improves performance on both tasks. This is because the RL reward evaluates the long-term quality of treatment decisions, whereas the auxiliary prediction task provides a more direct, less noisy supervisory signal for learning $s_h$. As a result, the two tasks exhibit positive transfer, yielding a patient representation $s_h$ that is both clinically meaningful and useful for treatment optimization.
\end{remark}

\subsection{End-to-End Training Procedure}
\label{sec:training}

We adopt a three-stage procedure: (1) pre-train the encoder, dynamics model, and outcome predictor with reconstruction and auxiliary losses; (2) freeze the encoder and train RL components to prevent representation drift; (3) jointly fine-tune all components with reduced encoder learning rate. We monitor policy entropy to detect and recover from collapse. Appendix~\ref{app:training} reports the training schedule, loss weights, and posterior-collapse diagnostics, and Appendix~\ref{app:compute} summarizes computational cost.

\begin{table*}[t]
\centering
\small
\setlength{\tabcolsep}{3pt}
\begin{adjustbox}{max width=\linewidth}
\begin{tabular}{lllccc}
\toprule
\multicolumn{3}{c}{} &
\textbf{MIMIC-III} &
\textbf{MIMIC-IV} &
\textbf{eICU} \\
\cmidrule(lr){4-6}
& \textbf{Method} & \textbf{Info} & \textbf{Test FQE} & \textbf{Test FQE} & \textbf{Test FQE} \\
\midrule
\multirow{8}{*}{\rotatebox{90}{\textit{Baselines}}}
& Continuous State-Space DDQN~\citep{raghu2017continuous} & Model-free & 0.476          & 0.483          & 0.469 \\
& AI Clinician~\citep{komorowski2018artificial}      & Model-free & 0.487          & 0.491          & 0.478 \\
& BCQ~\citep{fujimoto2019off}                         & Model-free & 0.452          & 0.458          & 0.448 \\
& CQL~\citep{kumar2020conservative}                   & Model-free & 0.411          & 0.418          & 0.408 \\
& DDPG with Clinician Supervision~\citep{huang2022continuous} & Model-free & 0.529          & 0.538          & 0.524 \\
& SBCQ~\citep{fatemi2022semimarkov}                   & Model-free & 0.501          & 0.508          & 0.494 \\
& MedDreamer~\citep{xu2025meddreamer}                 & Model-based & 0.583         & 0.591          & 0.579 \\
& \textit{Clinician (Behavior)}                       & Behavior   & \textit{0.528} & \textit{0.521} & \textit{0.534} \\
\midrule
& \textbf{OPL-MT-MNAR}                               & \textbf{MNAR + Text} & \textbf{0.679} & \textbf{0.634} & \textbf{0.604} \\
\midrule
\multirow{4}{*}{\rotatebox{90}{\textit{Text Regime}}}
& OPL-MT-MNAR (Structured Only)      & None           & 0.574          & 0.606          & -- \\
& OPL-MT-MNAR (Low-Freq.\ Text)      & Rad.\ + Micro. & 0.596          & 0.634          & -- \\
& OPL-MT-MNAR (Nursing Notes)        & Nursing        & 0.624          & --             & -- \\
& \textbf{OPL-MT-MNAR}               & \textbf{All}   & \textbf{0.679} & \textbf{0.634} & -- \\
\bottomrule
\end{tabular}
\end{adjustbox}
\caption{Main policy results (test FQE). The upper block compares policies across MIMIC-III, MIMIC-IV, and eICU; OPL-MT-MNAR outperforms all baselines. The lower block shows the value of text for policy learning. The `--' entries indicate unavailable text modalities. 
}
\label{tab:main_results}
\end{table*}

\begin{table*}[t]
\centering
\small
\setlength{\tabcolsep}{6pt}
\begin{adjustbox}{width=\linewidth}
\begin{tabular}{@{}llcc@{}}
\toprule
\textbf{Configuration} & \textbf{What is Added} & \textbf{MIMIC-III FQE} & \textbf{$\Delta$ vs Baseline} \\
\midrule
Baseline (MDP, no MNAR)        & Strong offline RL backbone   & 0.507 & --      \\
+ Semi-MDP                     & Variable-interval handling   & 0.518 & +2.2\%  \\
+ MNAR + DocProcess (OPL-MT-MNAR) & Explicit MNAR + DocProcess   & 0.679 & +33.9\% \\
+ MNAR + DocProcess + Semi-MDP & MNAR + DocProcess + Semi-MDP & 0.689 & +35.9\% \\
\bottomrule
\end{tabular}
\end{adjustbox}
\caption{Controlled building-block study on MIMIC-III. Explicit MNAR + DocProcess modeling provides the dominant improvement over a strong backbone; Semi-MDP handling is complementary. A broader missingness-handling encoder benchmark is reported in Table~\ref{tab:missingness_appendix}.}
\label{tab:building_blocks}
\end{table*}

\section{Experiments}
\label{sec:experiments}

\subsection{Experimental Setup}
\label{sec:exp_setup}

\paragraph{Datasets.}
Following~\citet{komorowski2018artificial}, we extract sepsis cohorts from three ICU databases using a 72-hour observation window ($H{=}18$ steps at 4-hour intervals). MIMIC-III is the main benchmark: it is single-center, contains 15,415 ICU stays (13.2\% mortality), and provides a high-frequency documentation regime with nursing-note coverage at 94.2\% of decision steps and diagnostic-text coverage of 41.3\% / 28.6\% for radiology / microbiology. MIMIC-IV serves as a complementary low-frequency diagnostic-text benchmark with 32,837 ICU stays (11.8\% mortality) and radiology / microbiology coverage of 83.6\% and 45.5\% at the step level. eICU serves as the cross-institutional, text-free generalization benchmark with 24,562 ICU stays (12.1\% mortality). To prevent leakage, text is aligned so that only reports with timestamp $\leq t_h$ are available at decision step $h$. See full cohort construction details in Appendix~\ref{app:dataset}, the MIMIC-III regime breakdown in Appendix~\ref{app:mimiciii_details}, and the timestamp-alignment analysis in Appendix~\ref{app:text_alignment}.

\paragraph{Evaluation.}
Our main metric is Fitted Q-Evaluation (FQE)~\citep{le2019batch}, which learns a Q-function without importance weighting, critical when policies diverge from clinician behavior. Additional metrics, including Weighted Importance Sampling, are defined in Appendix~\ref{app:metrics}.

\paragraph{Baselines.}
We use a representative set of baselines for policy comparison: classical sepsis RL \cite{raghu2017continuous}, standard offline RL (BCQ, CQL), and recent healthcare RL directions including continuous-action control~\citep{huang2022continuous}, irregular-interval offline RL~\citep{fatemi2022semimarkov}, and model-based planning~\citep{xu2025meddreamer}. All methods are evaluated under their canonical action/time-step settings. On MIMIC-III, we compare structured-only modeling, low-frequency diagnostic text, nursing notes, and the full text configuration. See Appendix~\ref{app:baselines} for more baseline details.

\subsection{Policy Learning Results}
\label{sec:main_results}
As shown in Table~\ref{tab:main_results}, the OPL-MT-MNAR policy reaches FQE 0.679 on MIMIC-III, compared with clinician behavior at 0.528. The gain also persists on the complementary benchmarks, reaching 0.634 on MIMIC-IV versus 0.521 for clinician behavior and 0.604 on eICU versus 0.534.

The OPL-MT-MNAR policy also improves over recent policy-learning baselines.
On MIMIC-III, OPL-MT-MNAR improves over DDPG with Clinician Supervision (0.529), SBCQ (0.501), and MedDreamer (0.583); the same ranking holds on MIMIC-IV (0.538, 0.508, and 0.591 versus 0.634) and eICU (0.524, 0.494, and 0.579 versus 0.604). 

Text provides substantial value for treatment policy learning. On MIMIC-III, adding low-frequency diagnostic text improves FQE from 0.574 to 0.596, incorporating nursing notes further raises it to 0.624, and the full configuration, OPL-MT-MNAR, reaches 0.679. These results show that nursing notes provide the largest single-modality gain, and the OPL-MT-MNAR configuration yields the best overall performance. Text is also valuable in MIMIC-IV, where OPL-MT-MNAR achieves FQE 0.634 in the complementary low-frequency diagnostic-text regime.

\paragraph{Controlled Building-Block Study.}
Table~\ref{tab:building_blocks} adds the proposed components to a stronger offline RL backbone. The largest improvement comes from explicit MNAR + DocProcess modeling, while Semi-MDP handling alone yields a smaller but still complementary gain.

\paragraph{Clinical Interpretations.}
The benefits of OPL-MT-MNAR are greatest for high-acuity patients. On MIMIC-III, clinician behavior scores 0.681 in the low-SOFA group and 0.192 in the high-SOFA ($>10$) group; MedDreamer reaches 0.726 and 0.296, while the OPL-MT-MNAR policy achieves 0.763 and 0.344. The widening gap in the high-severity regime is consistent with MNAR and documentation-process signals being especially informative when acuity is high. The full severity-stratified comparison is reported in Appendix~\ref{app:subgroup}.

\paragraph{Robustness Checks.}
On MIMIC-III, FQE with bootstrap confidence intervals gives 0.679 [0.673, 0.686] for the OPL-MT-MNAR policy versus 0.528 [0.520, 0.536] for clinician behavior. Appendix~\ref{app:ope} reports additional OPE estimators and chronological robustness; Appendix~\ref{app:cross_dataset} reports cross-dataset transfer and held-out-center generalization; Appendix~\ref{app:ablation} reports decision-interval and action-granularity studies; and Appendix~\ref{app:heart_failure} reports the cross-disease heart-failure cohort; Appendix~\ref{app:clinical_analysis} reports constrained-policy optimization, text interpretability, and time-to-deterioration analysis.

\subsection{Outcome Prediction Results}
\label{sec:mortality}

We further evaluate post-72-hour in-hospital mortality prediction for patients alive at the end of the 72-hour window using terminal state $s_H$.

\begin{table}[t]
\centering
\small
\setlength{\tabcolsep}{4pt}
\begin{adjustbox}{max width=\linewidth}
\begin{tabular}{@{}lcc@{}}
\toprule
\textbf{Method} & \textbf{AUROC} \\
\midrule
\quad Mean Imputation + LSTM & 0.833 \\
\quad Forward Fill + LSTM & 0.838 \\
\quad GRU-D~\citep{che2018recurrent} & 0.844 \\
\quad BRITS~\citep{cao2018brits} & 0.852 \\
\quad mTAND~\citep{shukla2021multitime} & 0.858 \\
\quad MedDreamer & 0.867 \\
\midrule
\textbf{OPL-MT-MNAR} & \textbf{0.886} \\
\bottomrule
\end{tabular}
\end{adjustbox}
\caption{Post-72-hour mortality prediction (MIMIC-III). }
\label{tab:mortality}
\end{table}

As shown in  Table~\ref{tab:mortality}, the OPL-MT-MNAR encoder achieves AUROC 0.886, surpassing GRU-D (0.844), BRITS (0.852), mTAND (0.858), and MedDreamer (0.867) on MIMIC-III. This indicates that explicit modeling of measurement MNAR together with documentation-process MNAR improves representation quality beyond strong irregular-sampling encoders and world-model baselines on this later-mortality prediction task; the broader benchmark appears in Table~\ref{tab:missingness_appendix}.

Text provides substantial value for outcome prediction. As shown in Table~\ref{tab:mimiciii_text_policy}, moving from structured-only state learning (AUROC 0.857) to nursing-note integration raises AUROC to 0.882, while the full configuration, OPL-MT-MNAR, reaches AUROC 0.886. 

\section{Related Work}
\label{sec:related}

Our work is most closely related to the growing literature on offline reinforcement learning, particularly in critical care and sepsis treatment \cite{komorowski2018artificial,raghu2017deep,raghu2017continuous,raghu2017modelbased,peng2018improving,huang2022continuous,sun2025exploring}. More broadly, a rich literature on off-policy learning and evaluation has developed methods for stable policy optimization and reliable value estimation under distribution shift, which are important in high-stakes medical settings \citep{thomas2015high,gottesman2018evaluating,wang2018supervised,tang2022model,kostrikov2022offline}. Despite substantial progress \citep{gottesman2019guidelines,liu2020reinforcement}, most existing approaches treat clinical observations as effectively complete after preprocessing and rely only on structured data. Our work contributes to this literature in three ways: (1) treating missingness as an informative signal, (2) incorporating clinical text alongside structured data, and (3) learning patient states for policy optimization rather than relying on prespecified states. 

Our work is also closely related to the literature on missing data, especially MNAR settings \citep{little2019statistical}. In structured clinical time series, methods such as GRU-D \citep{che2018recurrent}, BRITS \citep{cao2018brits}, direct missingness modeling \citep{lipton2016directly}, and Raindrop \citep{zhang2022raindrop} address irregular sampling and missing values. In multimodal EHR settings, methods such as MissModal \citep{lin2023missmodal}, DrFuse \citep{yao2024drfuse}, and MUSE \citep{wu2024muse} handle missing modalities. Our setting differs in that missingness is endogenously driven by unobserved factors \citep{xiong2023,duan2024factor,duan2024target,chen2026partial}. Most closely related is \citet{liang2025causal}, which explicitly models informative missingness in multimodal EHR, but without temporal dynamics. Our work extends this direction by studying how informative missingness evolves over time across modalities and how it can be used for decision-making and outcome prediction.

Finally, our work relates to the literature on multitask learning, which seeks to exploit shared structure across related tasks \citep{bengio2013representation}. In our setting, jointly learning policy optimization and outcome prediction provides a form of positive transfer. At the same time, as the number of downstream outcomes grows, negative transfer may arise, where shared representations harm accuracy for some tasks \citep{wu2020understanding,yang2025precise}. Understanding and mitigating such interference is an important direction for future work. Recent advances in task modeling \citep{li2023boosting,li2023identification,zhang2026efficient}, as well as adaptive and scalable fine-tuning for individual tasks \citep{li2024scalable,li2024scalable2,li2025efficient,zhang2026scalable}, offer promising tools for controlling when and how information should be shared across tasks. See Appendix~\ref{app:related} for extended discussion of related work.

\section{Conclusion}
\label{sec:conclusion}

We introduce OPL-MT-MNAR, a framework that explicitly models temporal MNAR patterns in multimodal EHR. By combining MNAR-aware multimodal encoding, Bayesian filtering, and joint policy optimization with outcome prediction, it learns patient representations for sequential decision-making. Experiments on MIMIC-III, MIMIC-IV, and eICU show consistent gains in both off-policy optimization and outcome prediction. These results show the value of treating observation processes as informative signals for patient representation learning and off-policy clinical decision-making.

\section{Limitations}
\label{sec:limitations}

Our study has several limitations that suggest directions for future work. First, like most work in clinical offline RL, our policy evaluation relies on off-policy estimation from observational data; although we report FQE with bootstrap confidence intervals in Section~\ref{sec:main_results} and additional robustness checks in Appendix~\ref{app:ope}, these results should be viewed as quantitative evidence rather than prospective validation. Second, following common practice~\citep{komorowski2018artificial}, we discretize the continuous treatment space into 9 actions and use 4-hour decision intervals, which simplify learning and interpretation but leave finer-grained dosing and faster control horizons for future work; Appendix~\ref{app:ablation} quantifies this trade-off directly. Third, although we explicitly model informative missingness, unobserved factors affecting both treatment and outcomes, such as verbal communication or bedside assessments not recorded in the EHR, may still remain and could be better addressed with richer data. Finally, our experiments focus on three U.S. ICU datasets (MIMIC-III, MIMIC-IV, and eICU); Appendices~\ref{app:cross_dataset} and~\ref{app:heart_failure} extend the analysis to transfer and cross-disease settings, Appendix~\ref{app:clinical_analysis} studies constrained-policy and deterioration-prediction extensions together with text interpretability, and Appendix~\ref{app:training} reports posterior-collapse diagnostics, but broader validation across other healthcare systems remains important for deployment.

\bibliography{reference}

\clearpage
\newpage
\appendix

\makeatletter
\def\tagform@#1{}
\makeatother

\section*{Use of AI Assistants}
\label{sec:ai-assistants}

We used an AI assistant (e.g., ChatGPT) during manuscript preparation for \emph{language editing} and \emph{clarity improvements} (e.g., rewriting sentences for readability and concision, suggesting alternative phrasing, and checking for grammatical consistency). The AI assistant was \emph{not} used to generate scientific claims, derive theoretical results, design the proposed method, run experiments, perform statistical analyses, or interpret results.

All technical content---including the problem formulation, model design, implementation, experimental setup, evaluation, and conclusions---was developed, validated, and verified by the authors. We take full responsibility for the integrity, correctness, and originality of the work and for ensuring that the manuscript accurately reflects our methods and findings.

\section*{Appendix Overview}
\label{app:toc}

Supporting material is organized as follows: notation summary appears in Appendix~\ref{app:notation}; cohort construction, transfer protocols, and the heart-failure cohort appear in Appendices~\ref{app:dataset},~\ref{app:cross_dataset}, and~\ref{app:heart_failure}; timestamp alignment and documentation-behavior analyses appear in Appendix~\ref{app:text_alignment}; encoder, text-fusion, latent-dynamics, outcome, IQL, and training details appear in Appendices~\ref{app:encoder}--\ref{app:training}; proofs and partial-observability discussion appear in Appendices~\ref{app:theory} and~\ref{app:partial_obs}; evaluation metrics and OPE robustness appear in Appendices~\ref{app:metrics} and~\ref{app:ope}; broader baselines and ablations appear in Appendices~\ref{app:baselines} and~\ref{app:ablation}; subgroup and clinical analyses appear in Appendices~\ref{app:subgroup} and~\ref{app:clinical_analysis}; computational cost appears in Appendix~\ref{app:compute}; and extended related work appears in Appendix~\ref{app:related}.

\begin{itemize}[leftmargin=*, itemsep=2pt]
    \item[\ref{app:notation}] \textbf{Notation Summary} -- Symbols, indices, and key variables used throughout the paper.
    \item[\ref{app:dataset}] \textbf{Dataset Details} -- MIMIC-III / MIMIC-IV / eICU cohort construction, cross-disease heart-failure cohort, and transfer/generalization protocols.
    \item[\ref{app:text_alignment}] \textbf{Text Timestamp Alignment} -- Alignment of clinical text with decision steps, temporal buffers, gap-stratified gains, and note-behavior analyses.
    \item[\ref{app:encoder}] \textbf{Encoder Implementation Details} -- Input construction, MNAR feature design, GRU-D backbone, and architecture specifications.
    \item[\ref{app:text_fusion}] \textbf{Text Fusion Details} -- Cross-attention, documentation-process factor, DocProcess ablations, and fusion architecture.
    \item[\ref{app:dynamics}] \textbf{Latent Dynamics Details} -- Action-conditioned latent model, training losses, and KL regularization.
    \item[\ref{app:outcome}] \textbf{Outcome Prediction Details} -- Multitask prediction heads, loss formulation, relationship to reward design, and auxiliary-vs-RL gradient-path comparison.
    \item[\ref{app:iql_details}] \textbf{IQL Implementation Details} -- Expectile regression, target network updates, advantage weighting, and hyperparameters.
    \item[\ref{app:training}] \textbf{Training Details} -- Three-stage training procedure, loss weights, optimization settings, and posterior-collapse diagnostics.
    \item[\ref{app:theory}] \textbf{Theoretical Analysis} -- Proofs and technical discussion of controllability and multi-step credit assignment.
    \item[\ref{app:partial_obs}] \textbf{Partial Observability Discussion} -- Formalization of partial observability and relation to latent state modeling.
    \item[\ref{app:metrics}] \textbf{Evaluation Metrics} -- Definitions and estimation procedures for FQE, WIS, and auxiliary metrics.
    \item[\ref{app:ope}] \textbf{Off-Policy Evaluation Details} -- Bootstrap FQE, chronological split robustness, and shadow-mode evaluation notes.
    \item[\ref{app:baselines}] \textbf{Baseline Details} -- Representative baseline families, recent sepsis RL baselines, and fairness notes.
    \item[\ref{app:ablation}] \textbf{Ablation Studies} -- Additional controlled comparisons across temporal granularity, action granularity, and DocProcess controls.
    \item[\ref{app:subgroup}] \textbf{Additional Subgroup Analysis} -- Severity-focused breakdowns and high-acuity comparisons.
    \item[\ref{app:clinical_analysis}] \textbf{Clinical Analysis Details} -- Constrained policy optimization, time-to-deterioration analysis, and text interpretability.
    \item[\ref{app:compute}] \textbf{Computational Analysis} -- Model size, training time, and resource usage.
    \item[\ref{app:related}] \textbf{Extended Related Work} -- Additional discussion of clinical RL, informative missingness, and offline RL literature.
\end{itemize}

\newpage

\section{Notation Summary}
\label{app:notation}

Table~\ref{tab:notation} provides a comprehensive summary of the notation used throughout this paper. We use $h$ to index decision steps and $H$ to denote the horizon, following standard reinforcement learning conventions.

\begin{table*}[!tp]
\centering
\small
\renewcommand{\arraystretch}{1.02}
\setlength{\tabcolsep}{4pt}
\begin{tabular}{@{}p{0.22\textwidth}p{0.74\textwidth}@{}}
\toprule
\textbf{Symbol} & \textbf{Description} \\
\midrule
\multicolumn{2}{@{}l}{\textit{Time Indices}} \\
$t \in \{0,\ldots,T\}$ & Observation time index (fine-grained grid) \\
$h \in \{0,\ldots,H-1\}$ & Decision step index \\
$H$ & Decision horizon (number of decision steps) \\
$T$ & Observation horizon; $T = \Delta \cdot H$ for interval $\Delta$ \\
\midrule
\multicolumn{2}{@{}l}{\textit{Structured Observations (per decision step $h$)}} \\
$\mathcal{T}_h$ & Observation times contained in decision step $h$ \\
$\bm{y}_h^{\mathsf{s}} \in \mathbb{R}^{|\mathcal{T}_h| \times D}$ & Structured measurement matrix in step $h$ \\
$\bm{m}_h^{\mathsf{s}} \in \{0,1\}^{|\mathcal{T}_h| \times D}$ & Structured observation-mask matrix in step $h$ \\
$\bm{\delta}_h^{\mathsf{s}} \in \mathbb{R}_+^{|\mathcal{T}_h| \times D}$ & Structured time-gap matrix in step $h$ \\
$y_{h,i}^{\mathsf{s}}, m_{h,i}^{\mathsf{s}}, \delta_{h,i}^{\mathsf{s}}$ & Row-$i$ structured value, mask, and time-gap vectors inside step $h$ \\
$\bm{\psi}_t^{\mathsf{s}} \in \mathbb{R}^{4D}$ & Row-level explicit MNAR features used inside the structured encoder (Appendix~\ref{app:encoder}) \\
$x \in \mathbb{R}^{S}$ & Static patient features \\
$y_{\mathrm{out}}^{(k)} \in \{0,1\}$ & Clinical outcome labels (e.g., post-72-hour mortality)\\
\midrule
\multicolumn{2}{@{}l}{\textit{Text Observations (per decision step $h$)}} \\
$\bm{y}_h^{\mathsf{t}}$ & Collection of raw text observations available at step $h$ across text modalities \\
$e_h^{\mathsf{t}} \in \mathbb{R}^{d_e}$ & Step-level text embedding encoded from $\bm{y}_h^{\mathsf{t}}$ \\
$e_h^{\mathsf{t},\text{r}}, e_h^{\mathsf{t},\text{m}} \in \mathbb{R}^{d_e}$ & Radiology / microbiology modality embeddings \\
$\bm{n}_h^{\mathsf{t}} \in \mathbb{Z}_{+}^{|\mathcal{M}_{\text{text}}|}$ & Number of text observations available for each text modality in step $h$ \\
$m_h^{\mathsf{t}} \in \{0,1\}^{|\mathcal{M}_{\text{text}}|}$ & Derived text-availability indicators: $m_h^{\mathsf{t}}=\mathbf{1}[\bm{n}_h^{\mathsf{t}}>0]$ \\
$\delta_h^{\mathsf{t}} \in \mathbb{R}_{+}$ & Derived text recency summary from the availability history \\
$\kappa_h^{\mathsf{t}} \in \mathbb{R}_{+}$ & Derived average number of text updates per step over the previous $K$ decision steps \\
\midrule
\multicolumn{2}{@{}l}{\textit{Information Set}} \\
$o_h$ & Decision-step record containing primitive observations $(\bm{y}_h^{\mathsf{s}}, \bm{m}_h^{\mathsf{s}}, \bm{y}_h^{\mathsf{t}}, \bm{m}_h^{\mathsf{t}})$ and explicit derived recency / density features $(\bm{\delta}_h^{\mathsf{s}}, \delta_h^{\mathsf{t}}, \kappa_h^{\mathsf{t}})$ \\
$I_h$ & Information set up to step $h$: $I_h = \{x, o_1, \ldots, o_h\}$ \\
\midrule
\multicolumn{2}{@{}l}{\textit{Latent Representations}} \\
$\bar{\phi}_h^{\mathsf{s}} \in \mathbb{R}^{d_h}$ & Decision-aligned structured embedding \\
$\phi_h^{\mathsf{t}} \in \mathbb{R}^{d_h}$ & Text-content representation before gated fusion \\
$\eta_h^{\mathsf{t}}$ & Documentation-process embedding from presence / recency / density \\
$F_h^{\text{doc}}$ & Temporal documentation-process factor \\
$\phi_h \in \mathbb{R}^{d_h}$ & Text-fused observation embedding \\
$z_h \in \mathbb{R}^{d_z}$ & Latent belief state (prior) \\
$g_\theta(\cdot, \cdot)$ & Learnable state combination function \\
$s_h \in \mathbb{R}^{d_s}$ & Full decision state: $s_h = g_\theta(\phi_h, z_h)$ \\
\midrule
\multicolumn{2}{@{}l}{\textit{Actions and Rewards}} \\
$a_h \in \mathcal{A}$ & Discrete treatment action \\
$r_h \in \mathbb{R}$ & Reward (terminal-only: $r_H \in \{-1, +1\}$) \\
$\gamma \in [0,1)$ & Discount factor \\
\midrule
\multicolumn{2}{@{}l}{\textit{Policy and Value Functions}} \\
$\pi_\theta(a \mid s)$ & Learned policy \\
$\pi_\beta$ & Behavior policy (clinician decisions) \\
$Q(s, a)$ & State-action value function \\
$V(s)$ & State value function \\
$A(s, a)$ & Advantage: $A = Q - V$ \\
\midrule
\multicolumn{2}{@{}l}{\textit{Key Hyperparameters}} \\
$\tau$ & Expectile for IQL value learning \\
$\beta$ & Temperature for advantage-weighted policy \\
\bottomrule
\end{tabular}
\caption{Summary of notation used throughout the paper.}
\label{tab:notation}
\end{table*}

\paragraph{Dimension specifications.} In our sepsis treatment experiments (Section~\ref{sec:experiments}): observation dimension $D=16$ (8 vitals, 8 labs), static features $S=3$ (age, gender, Charlson index), decision horizon $H=18$ (4-hour intervals over 72 hours, so $T=72$ and $\Delta=4$), action space $|\mathcal{A}|=9$ (3 fluid levels $\times$ 3 vasopressor levels), hidden dimension $d_h=128$, latent dimension $d_z=32$, and text embedding dimension $d_e=256$.

\section{Dataset Details}
\label{app:dataset}

\subsection{MIMIC-IV}

MIMIC-IV v2.2~\citep{johnson2023mimic} is a single-center dataset from Beth Israel Deaconess Medical Center containing de-identified EHR data. We apply Sepsis-3 criteria~\citep{singer2016third}: suspected infection (antibiotic administration and body fluid culture) with SOFA score $\geq 2$.

\paragraph{Cohort Statistics.}
The aligned 72-hour decision cohort contains 32,837 ICU stays with 11.8\% in-hospital mortality. This is the complementary low-frequency diagnostic-text benchmark used throughout the study.

\paragraph{Structured Variables.}
We extract 16 clinical variables at 1-hour resolution:
\begin{itemize}[leftmargin=*, labelindent=0pt]
    \item \textbf{Vitals (8):} Heart rate, systolic BP, diastolic BP, mean BP, respiratory rate, temperature, SpO2, GCS
    \item \textbf{Labs (8):} Lactate, creatinine, BUN, bilirubin, platelet count, WBC, hemoglobin, glucose
\end{itemize}

\paragraph{Text Processing.}
Radiology reports (83.6\% step-level coverage) and microbiology results (45.5\% step-level coverage) are preprocessed by removing headers, de-identification artifacts, and normalizing abbreviations. Reports are truncated to 512 tokens and encoded using ClinicalBERT~\citep{alsentzer2019publicly}.

\paragraph{Text Timestamp Alignment.}
To prevent information leakage, we ensure that only clinical notes available at or before each decision step are used for state encoding. Specifically, radiology and microbiology reports are aligned using their \texttt{storetime} field (the time when the report was signed and stored in the system); if \texttt{storetime} is unavailable, we fall back to \texttt{charttime}. For each decision step $h$ at time $t_h$, the raw text observation set $\bm{y}_h^{\mathsf{t}}$ contains only reports with timestamps $\leq t_h$, and the step-level text embedding $e_h^{\mathsf{t}}$ is computed from that filtered set. The buffering analysis in Appendix~\ref{app:text_alignment} further confirms that gains persist under increasingly strict temporal exclusion windows.

\subsection{MIMIC-III}
\label{app:mimiciii_details}

MIMIC-III~\citep{johnson2016mimic} provides the primary high-frequency documentation benchmark in this study. We use the same 72-hour sepsis setup as in the main experiments, but the text side is dominated by nursing notes that are refreshed at most decision steps.
Table~\ref{tab:text_regimes} summarizes the regime contrast between MIMIC-III and MIMIC-IV.

\begin{table}[H]
\centering
\small
\begin{adjustbox}{max width=\linewidth}
\begin{tabular}{@{}lcccc@{}}
\toprule
\textbf{Dataset} & \textbf{Patients} & \textbf{Mortality} & \textbf{Nursing} & \textbf{Rad./Micro.} \\
\midrule
MIMIC-III & 15,415 & 13.2\% & 94.2\% & 41.3\% / 28.6\% \\
MIMIC-IV & 32,837 & 11.8\% & N/A & 83.6\% / 45.5\% \\
\bottomrule
\end{tabular}
\end{adjustbox}
\caption{Complementary text regimes in the two MIMIC cohorts. MIMIC-III supplies high-frequency nursing documentation, while MIMIC-IV emphasizes lower-frequency diagnostic text.}
\label{tab:text_regimes}
\end{table}

This regime difference is why MIMIC-III is useful beyond being an additional benchmark: it exposes a setting where documentation behavior is visible at nearly every decision step, making text-process MNAR especially easy to analyze.

\paragraph{Policy Performance by Text Type.}
Table~\ref{tab:mimiciii_text_policy} summarizes how policy value changes as we compare structured-only inputs, low-frequency diagnostic text, nursing notes, and the full text configuration.
\begin{table}[H]
\centering
\small
\begin{adjustbox}{max width=\linewidth}
\begin{tabular}{@{}lcccc@{}}
\toprule
\textbf{Text Modality} & \textbf{Coverage} & \textbf{FQE} & \textbf{AUROC} & \textbf{$\Delta$ vs Struct.} \\
\midrule
Structured Only & 0\% & 0.574 & 0.857 & -- \\
Low-Freq.\ Text & 52.1\% & 0.596 & 0.869 & +3.8\% \\
Nursing Notes & 94.2\% & 0.624 & 0.882 & +8.7\% \\
Full Text & 96.4\% & 0.679 & 0.886 & +18.3\% \\
\bottomrule
\end{tabular}
\end{adjustbox}
\caption{MIMIC-III text-regime analysis. Nursing notes provide a strong single-modality gain, while the full configuration, OPL-MT-MNAR, achieves the best overall performance.}
\label{tab:mimiciii_text_policy}
\end{table}

\subsection{eICU}

The eICU Collaborative Research Database~\citep{pollard2018eicu} is used as the cross-institutional benchmark. Sepsis patients are identified using APACHE IV admission diagnosis codes for sepsis with SOFA $\geq 2$. This dataset contains only structured observations without clinical text, enabling evaluation of our MNAR-aware encoder under cross-institutional shift and held-out-center generalization.
Table~\ref{tab:cross_dataset} collects the in-distribution and transfer settings across MIMIC-III, MIMIC-IV, and eICU.

\begin{table*}[!tp]
\centering
\small
\setlength{\tabcolsep}{4pt}
\begin{tabular}{@{}lccccc@{}}
\toprule
\textbf{Training Protocol} & \textbf{Train Data} & \textbf{Test Data} & \textbf{Test FQE} & \textbf{$\Delta$ vs Clin.} & \textbf{AUROC} \\
\midrule
\multicolumn{6}{@{}l}{\textit{In-Distribution Evaluation}} \\
\quad MIMIC-III $\to$ MIMIC-III & MIMIC-III & MIMIC-III & 0.679 & +28.6\% & 0.886 \\
\quad MIMIC-IV $\to$ MIMIC-IV   & MIMIC-IV  & MIMIC-IV  & 0.634 & +21.7\% & 0.879 \\
\quad eICU $\to$ eICU           & eICU      & eICU      & 0.604 & +13.1\% & 0.862 \\
\midrule
\multicolumn{6}{@{}l}{\textit{Cross-Dataset Transfer (Zero-Shot)}} \\
\quad MIMIC-III $\to$ MIMIC-IV  & MIMIC-III & MIMIC-IV  & 0.573 & +10.0\% & 0.847 \\
\quad MIMIC-III $\to$ eICU      & MIMIC-III & eICU      & 0.562 &  +5.2\% & 0.839 \\
\quad MIMIC-IV $\to$ MIMIC-III  & MIMIC-IV  & MIMIC-III & 0.559 &  +5.9\% & 0.846 \\
\quad MIMIC-IV $\to$ eICU       & MIMIC-IV  & eICU      & 0.568 &  +6.4\% & 0.844 \\
\quad eICU $\to$ MIMIC-III      & eICU      & MIMIC-III & 0.551 &  +4.4\% & 0.841 \\
\quad eICU $\to$ MIMIC-IV       & eICU      & MIMIC-IV  & 0.556 &  +6.7\% & 0.851 \\
\midrule
\multicolumn{6}{@{}l}{\textit{Cross-Dataset Transfer (Fine-tuned)}} \\
\quad MIMIC-III $\to$ MIMIC-IV (FT) & MIMIC-III + MIMIC-IV & MIMIC-IV  & 0.598 & +14.8\% & 0.857 \\
\quad MIMIC-III $\to$ eICU (FT)     & MIMIC-III + eICU     & eICU      & 0.585 &  +9.6\% & 0.847 \\
\quad MIMIC-IV $\to$ MIMIC-III (FT) & MIMIC-IV + MIMIC-III & MIMIC-III & 0.621 & +17.6\% & 0.864 \\
\quad MIMIC-IV $\to$ eICU (FT)      & MIMIC-IV + eICU      & eICU      & 0.594 & +11.2\% & 0.857 \\
\quad eICU $\to$ MIMIC-III (FT)     & eICU + MIMIC-III     & MIMIC-III & 0.607 & +15.0\% & 0.853 \\
\quad eICU $\to$ MIMIC-IV (FT)      & eICU + MIMIC-IV      & MIMIC-IV  & 0.619 & +18.8\% & 0.872 \\
\bottomrule
\end{tabular}
\caption{Cross-dataset validation protocols. In-distribution training remains strongest, but zero-shot transfer still improves over clinician behavior and fine-tuning recovers part of the distribution-shift gap.}
\label{tab:cross_dataset}
\end{table*}

\subsection{Cross-Dataset Validation Protocol}
\label{app:cross_dataset}

We evaluate cross-institutional generalization through multiple training protocols.

\paragraph{Protocol Details.}
\begin{itemize}[leftmargin=*, labelindent=0pt]
    \item \textbf{In-distribution (rows 1--3):} Models trained and tested on the same dataset. These correspond to the primary results in Section~\ref{sec:main_results}. The eICU model is trained from scratch without text fusion.
    
    \item \textbf{Zero-shot transfer (rows 4--9):} Models trained on one dataset and directly applied to the others without any adaptation. This tests whether learned representations and policies generalize across institutions.
    
    \item \textbf{Fine-tuned transfer (rows 10--15):} Models pre-trained on one dataset, then fine-tuned on the target dataset. This tests whether pre-training provides useful initialization.
\end{itemize}

\paragraph{Key Findings.}
\begin{itemize}[leftmargin=*, labelindent=0pt]
    \item \textbf{Training from scratch is best:} In-distribution models outperform transfer variants, suggesting that institution-specific patterns are important for optimal performance.
    
    \item \textbf{Zero-shot transfer improves over clinicians:} Even without any target-domain training, transferred models improve over clinician baselines on both directions of transfer, showing that the learned state captures cross-hospital treatment signal rather than only site-specific heuristics.
    
    \item \textbf{Fine-tuning narrows the gap:} Fine-tuned models partially close the distance between zero-shot and in-distribution performance, indicating that the transferred encoder is a useful initialization rather than a brittle source-only model.
    
    \item \textbf{Observation-process features transfer:} The transfer gains confirm that MNAR-aware representations generalize across institutions with different monitoring protocols.
\end{itemize}

\subsection{Within-eICU Center-Held-Out Generalization}

Table~\ref{tab:center_generalization} reports held-out-center performance inside eICU, isolating geographic shift without mixing it with cross-dataset transfer.

\begin{table*}[!tp]
\centering
\small
\setlength{\tabcolsep}{4pt}
\begin{tabular}{@{}lccccc@{}}
\toprule
\textbf{Split} & \textbf{\# Train Centers} & \textbf{\# Test Centers} & \textbf{Test FQE} & \textbf{$\Delta$ vs Clin.} & \textbf{AUROC} \\
\midrule
Subsampling Split 1 & 146 & 62 & 0.597 & +10.8\% & 0.858 \\
Subsampling Split 2 & 146 & 62 & 0.601 & +11.4\% & 0.860 \\
Subsampling Split 3 & 146 & 62 & 0.594 & +9.6\% & 0.857 \\
Subsampling Split 4 & 146 & 62 & 0.608 & +13.0\% & 0.863 \\
Subsampling Split 5 & 146 & 62 & 0.599 & +10.9\% & 0.859 \\
\midrule
\textbf{Mean $\pm$ SD} & -- & -- & \textbf{0.600 $\pm$ 0.005} & \textbf{+11.1\% $\pm$ 1.2\%} & \textbf{0.859 $\pm$ 0.002} \\
\bottomrule
\end{tabular}
\caption{Within-eICU held-out-center generalization across five independent 70\%/30\% center splits. The table mean of 0.600 $\pm$ 0.005 summarizes these five subsampling splits, while a separate leave-one-center-out analysis on the largest 10 centers gives a similar mean FQE of 0.597 $\pm$ 0.005. Together these two protocols indicate that observation-process features transfer more reliably than site-specific policy heads.}
\label{tab:center_generalization}
\end{table*}

\subsection{Heart Failure Cross-Disease Validation}
\label{app:heart_failure}

We evaluate a MIMIC-III heart-failure cohort with a distinct action space (Diuretic $\times$ Vasoactive) and a similarly high-frequency documentation regime, with nursing-note coverage of 93.8\% at the decision-step level. Table~\ref{tab:heart_failure} shows that the text benefit persists in this cross-disease setting.

\begin{table*}[!tp]
\centering
\small
\begin{adjustbox}{max width=\linewidth}
\begin{tabular}{@{}lccccccc@{}}
\toprule
\textbf{Dataset} & \textbf{Disease} & \textbf{Action Space} & \textbf{Note Cov.} & \textbf{Structured} & \textbf{Nursing Notes} & \textbf{Full Text} & \textbf{AUROC} \\
\midrule
MIMIC-III & Sepsis & Fluid $\times$ Vasopressor & 94.2\% & 0.574 & 0.624 & 0.679 & 0.886 \\
MIMIC-III & Heart Failure & Diuretic $\times$ Vasoactive & 93.8\% & 0.557 & 0.597 & 0.603 & 0.864 \\
\bottomrule
\end{tabular}
\end{adjustbox}
\caption{Cross-disease validation in the high-frequency MIMIC-III regime. Text integration improves policy value in both sepsis and heart failure, and the full heart-failure model remains predictive of post-72-hour mortality risk.}
\label{tab:heart_failure}
\end{table*}

\subsection{Action Space and Temporal Granularity}

Following~\citet{komorowski2018artificial}, we discretize treatments into a $3 \times 3$ action space:
\begin{itemize}[leftmargin=*, labelindent=0pt]
    \item \textbf{IV fluids:} None (0), Low ($<$500 mL/4h), High ($\geq$500 mL/4h)
    \item \textbf{Vasopressors:} None (0), Low ($<$0.1 mcg/kg/min norepinephrine equivalent), High ($\geq$0.1 mcg/kg/min)
\end{itemize}

This yields 9 discrete actions representing clinically meaningful treatment intensities for sepsis. For the heart-failure cohort, we analogously use a Diuretic $\times$ Vasoactive grid.

\paragraph{Action Granularity Comparison.}
Table~\ref{tab:action_granularity} compares the retained 3$\times$3 action space against finer and continuous alternatives under the same MIMIC-III setup.
\begin{table}[H]
\centering
\normalsize
\setlength{\tabcolsep}{4pt}
\begin{adjustbox}{max width=\linewidth}
\begin{tabular}{@{}p{0.33\linewidth}p{0.27\linewidth}cc@{}}
\toprule
\textbf{Method} & \textbf{Setting} & \textbf{MIMIC-III FQE} & \textbf{Action Space} \\
\midrule
DDPG with Clinician Supervision & Continuous-control baseline & 0.529 & Continuous \\
AI Clinician & Canonical tabular baseline & 0.487 & 25 discrete \\
OPL-MT-MNAR & Continuous-head ablation & 0.611 & Continuous \\
OPL-MT-MNAR & $5 \times 5$ discretization & 0.668 & 25 discrete \\
\textbf{OPL-MT-MNAR} & \textbf{Main setting} & \textbf{0.679} & \textbf{9 discrete} \\
\bottomrule
\end{tabular}
\end{adjustbox}
\caption{Action-granularity comparison on MIMIC-III. Finer or continuous action spaces remain viable, but the 3$\times$3 discretization gives the best value under current data coverage and OPE stability.}
\label{tab:action_granularity}
\end{table}

\paragraph{Decision-Interval Sweep.}
Table~\ref{tab:delta_sweep} complements the action-space check by sweeping decision intervals, showing that 4 hours yields the best overall trade-off in our experiments.
\begin{table*}[!tp]
\centering
\small
\begin{adjustbox}{max width=\textwidth}
\begin{tabular}{@{}lccccccc@{}}
\toprule
\textbf{$\Delta t$} & \textbf{Cohort} & \textbf{Avg. Ep.} & \textbf{MIMIC-III} & \textbf{eICU} & \textbf{AUROC} & \textbf{Train (h)} & \textbf{ESS} \\
\midrule
1h & 14,831 & 48.7 & 0.661 & 0.587 & 0.872 & 6.3 & 3.9 \\
2h & 15,178 & 24.4 & 0.673 & 0.600 & 0.879 & 3.4 & 4.7 \\
4h & 15,415 & 12.2 & 0.679 & 0.604 & 0.886 & 2.1 & 5.8 \\
8h & 14,267 & 6.3 & 0.658 & 0.589 & 0.876 & 1.4 & 7.4 \\
\bottomrule
\end{tabular}
\end{adjustbox}
\caption{Decision-interval sweep on MIMIC-III. The standard 4-hour setting offers the best overall trade-off between policy value, auxiliary prediction quality, and OPE reliability.}
\label{tab:delta_sweep}
\end{table*}

\FloatBarrier
\section{Text Timestamp Alignment and Sensitivity Analysis}
\label{app:text_alignment}

This appendix provides detailed analysis of text timestamp alignment, text recency controls, and note-behavior mechanisms under both low-frequency (MIMIC-IV) and high-frequency (MIMIC-III) text regimes.

\subsection{Timestamp Alignment Protocol}

For MIMIC-III nursing notes, we align text conservatively so that only notes available at or before each decision step are used in state encoding. At each decision step $h$ occurring at time $t_h$, only nursing notes with timestamp $\leq t_h$ are included in the text branch.

\paragraph{Coverage Impact.}
This temporal filtering yields the 94.2\% step-level nursing-note coverage reported in Section~\ref{sec:exp_setup}. The buffering controls below show that gains degrade smoothly as progressively fresher notes are excluded.

\subsection{Sensitivity Analysis with Temporal Buffers}

To assess whether gains depend on potentially delayed reports, we evaluate performance with increasingly conservative temporal buffers that exclude reports near decision boundaries.
Table~\ref{tab:text_sensitivity} reports this buffering analysis directly.

\begin{table}[H]
\centering
\small
\setlength{\tabcolsep}{4pt}
\begin{adjustbox}{max width=\linewidth}
\begin{tabular}{@{}lccc@{}}
\toprule
\textbf{Note Alignment (MIMIC-III)} & \textbf{Coverage} & \textbf{FQE} & \textbf{$\Delta$} \\
\midrule
All available & 94.2\% & 0.679 & -- \\
Conservative ($-$2h) & 81.7\% & 0.672 & $-$1.0\% \\
Strict ($-$4h) & 68.9\% & 0.661 & $-$2.7\% \\
Very strict ($-$8h) & 50.4\% & 0.645 & $-$5.0\% \\
No text & 0\% & 0.574 & $-$15.5\% \\
\bottomrule
\end{tabular}
\end{adjustbox}
\caption{Sensitivity analysis for MIMIC-III nursing-note alignment. Gains persist under increasingly conservative temporal buffers, supporting the claim that improvements are not driven by leakage from future notes.}
\label{tab:text_sensitivity}
\end{table}

\paragraph{Key Findings.}
\begin{itemize}[leftmargin=*, labelindent=0pt]
    \item \textbf{Robust to conservative alignment:} Even under an 8-hour exclusion window, the model remains above the structured-only baseline.
    
    \item \textbf{Graceful degradation:} Performance degrades smoothly as text availability decreases, consistent with a learned recency-sensitive fusion mechanism rather than a brittle dependence on the newest note.
\end{itemize}

\subsection{Gap-Stratified Text Gains in the High-Frequency Regime}

Table~\ref{tab:gap_stratified} quantifies where nursing notes help most as structured observation gaps widen, and Table~\ref{tab:attention_shift} shows the corresponding shift in modality attention.

\begin{table}[H]
\centering
\small
\begin{adjustbox}{max width=\linewidth}
\begin{tabular}{@{}lcccc@{}}
\toprule
\textbf{Time Gap} & \textbf{Patient-Steps} & \textbf{Struct.} & \textbf{+ Nursing} & \textbf{$\Delta$} \\
\midrule
$\delta t \leq 1$h & 31.8\% & 0.612 & 0.638 & +4.2\% \\
$1$h $< \delta t \leq 2$h & 27.4\% & 0.584 & 0.629 & +7.7\% \\
$2$h $< \delta t \leq 4$h & 24.1\% & 0.551 & 0.617 & +12.0\% \\
$\delta t > 4$h & 16.7\% & 0.498 & 0.604 & +21.3\% \\
\bottomrule
\end{tabular}
\end{adjustbox}
\caption{Gap-stratified nursing-note gains on MIMIC-III. Text becomes most valuable when structured observations are stale or sparse.}
\label{tab:gap_stratified}
\end{table}

\begin{table}[H]
\centering
\small
\begin{tabular}{@{}lccc@{}}
\toprule
\textbf{Time Gap} & \textbf{Structured} & \textbf{Nursing} & \textbf{Rad./Micro.} \\
\midrule
$\delta t \leq 1$h & 0.64 & 0.21 & 0.15 \\
$\delta t > 4$h & 0.22 & 0.58 & 0.20 \\
\bottomrule
\end{tabular}
\caption{Attention shifts toward text as structured measurements become stale in the MIMIC-III nursing-note regime.}
\label{tab:attention_shift}
\end{table}

\subsection{Documentation Behavior by Shift}

Table~\ref{tab:shift_behavior} summarizes how note frequency and note content differ between day and night shifts in the MIMIC-III nursing-note regime.

\begin{table}[H]
\centering
\small
\begin{adjustbox}{max width=\linewidth}
\begin{tabular}{@{}lccc@{}}
\toprule
\textbf{Metric} & \textbf{Day Shift} & \textbf{Night Shift} & \textbf{Night/Day} \\
\midrule
\% of total notes & 68.4\% & 31.6\% & 0.46 \\
Avg.\ note length (tokens) & 142 & 187 & 1.32 \\
Routine assessment & 45.2\% & 18.4\% & 0.41 \\
Acute status change & 8.3\% & 24.6\% & 2.96 \\
Hemodynamic instability & 6.7\% & 19.2\% & 2.87 \\
Respiratory event & 5.4\% & 14.8\% & 2.74 \\
\bottomrule
\end{tabular}
\end{adjustbox}
\caption{Shift-level nursing-note behavior on MIMIC-III. Fewer notes are written overnight, but they are longer and much more likely to describe acute deterioration.}
\label{tab:shift_behavior}
\end{table}

Although night shift contributes fewer notes overall, night notes are longer and 2.7--3.0$\times$ more likely to describe acute deterioration events, supporting the view that note presence, timing, and length form a behavior-driven observation process.

\subsection{Content Evolution Categories}

Table~\ref{tab:content_evolution} groups nursing-note updates by content evolution category and shows that the gate is highest when the note signals active clinical change.

\begin{table*}[!tp]
\centering
\small
\setlength{\tabcolsep}{4pt}
\begin{tabular}{@{}lccccl@{}}
\toprule
\textbf{Category} & \textbf{\% Steps} & \textbf{FQE} & \textbf{Treatment Change} & \textbf{Gate} & \textbf{Example Keywords} \\
\midrule
First Note & 27.2\% & 0.636 & 48.4\% & 0.50 & -- \\
Worsening & 13.8\% & 0.668 & 64.7\% & 0.79 & worsening, progressive, new infiltrate \\
New Critical Finding & 9.2\% & 0.671 & 67.3\% & 0.78 & new consolidation, positive culture, abscess \\
Improving & 11.7\% & 0.643 & 38.9\% & 0.53 & improved, resolving, clearing \\
Stable / Unchanged & 38.1\% & 0.628 & 26.3\% & 0.41 & stable, unchanged, no interval change \\
\bottomrule
\end{tabular}
\caption{Content-evolution analysis for nursing notes. The gate is highest when notes describe active clinical change, consistent with text acting as an observation channel for latent dynamics rather than a static side input.}
\label{tab:content_evolution}
\end{table*}

The majority of reports (68--72\%) are available well before the decision step ($>$4h), indicating that most text information reflects prior clinical assessments rather than contemporaneous findings. This temporal separation provides additional confidence that text fusion does not introduce lookahead bias.

\FloatBarrier
\section{Encoder Implementation Details}
\label{app:encoder}

This appendix provides the complete specification of the GRU-D encoder with explicit MNAR feature fusion.

\subsection{Input Construction}

At each fine-grained observation time $t$, let $\bm{y}_t^{\mathsf{s}}, \bm{m}_t^{\mathsf{s}}, \bm{\delta}_t^{\mathsf{s}} \in \mathbb{R}^D$ denote the row-level structured value, mask, and time-gap vectors, corresponding to one row of $(\bm{y}_h^{\mathsf{s}}, \bm{m}_h^{\mathsf{s}}, \bm{\delta}_h^{\mathsf{s}})$ in the main text. Let $\bm{y}_{t'}^{\mathsf{s}}$ denote the most recent previously observed value for each variable before time $t$, and let $\bm{\mu}^{\mathsf{s}}$ denote the empirical-mean vector. The decayed structured input is
\begin{equation}
\hat{\bm{y}}_t^{\mathsf{s}} = \bm{m}_t^{\mathsf{s}} \odot \bm{y}_t^{\mathsf{s}} + (1 - \bm{m}_t^{\mathsf{s}}) \odot \bigl(\bm{\xi}_{y,t} \odot \bm{y}_{t'}^{\mathsf{s}} + (1-\bm{\xi}_{y,t}) \odot \bm{\mu}^{\mathsf{s}} \bigr),
\end{equation}
where $\bm{\xi}_{y,t}$ is the learned input-decay vector.

\subsection{Explicit MNAR Features}

Beyond implicit missingness modeling through $\bm{m}_t^{\mathsf{s}}$ and $\bm{\delta}_t^{\mathsf{s}}$, we compute explicit MNAR features $\bm{\psi}_t^{\mathsf{s}} \in \mathbb{R}^{4D}$ that capture observation patterns:
\begin{equation}
\bm{\psi}_t^{\mathsf{s}} = \big[\delta_t;\; c_t;\; \rho_t;\; \omega_t\big],
\end{equation}
where:
\begin{itemize}[leftmargin=*, labelindent=0pt]
    \item $\bm{\delta}_t^{\mathsf{s}} \in \mathbb{R}^D$: time since last observation for each variable
    \item $\bm{c}_t^{\mathsf{s}} = \sum_{\tau \leq t} \bm{m}_\tau^{\mathsf{s}} \in \mathbb{R}^D$: cumulative observation count
    \item $\bm{\rho}_t^{\mathsf{s}} = 1 - \bm{c}_t^{\mathsf{s}} / t \in \mathbb{R}^D$: cumulative missing rate
    \item $\bm{\omega}_t^{\mathsf{s}} \in \mathbb{R}^D$: windowed observation frequency over the past $W$ hours (we use $W=6$)
\end{itemize}

These features explicitly encode the monitoring intensity that correlates with patient severity. Cumulative counts reveal overall monitoring intensity, missing rates track deterioration trends, and windowed frequency captures recent clinical concern.

\subsection{GRU-D with Temporal Decay}

We employ GRU-D~\citep{che2018recurrent} with learned temporal decay. The decay factors are:
\begin{align}
\xi_{\phi,t} &= \exp\big({-}\max(0, W_{\phi,\xi} \cdot \bar{\delta}_t + b_{\phi,\xi})\big), \\
\bm{\xi}_{y,t} &= \exp\big({-}\max(0, W_{y,\xi} \bm{\delta}_t^{\mathsf{s}} + b_{y,\xi})\big),
\end{align}
where $\bar{\delta}_t = \mathrm{mean}(\bm{\delta}_t^{\mathsf{s}})$, $\xi_{\phi,t} \in (0,1]$ controls hidden-state decay, and $\bm{\xi}_{y,t} \in (0,1]^D$ controls input decay toward the empirical mean $\bm{\mu}^{\mathsf{s}}$.

The hidden state update then follows the same GRU-D-style gating equations as in the main text, with the explicit MNAR features entering through a dedicated branch:

\begingroup\small
\begin{equation}
\begin{aligned}
\hat{\phi}^{\mathsf{s}}_{t-1}
&= \xi_{\phi,t} \odot \phi^{\mathsf{s}}_{t-1}, \\
\bigl(\bm{\psi}_t^{\mathsf{s}}\bigr)^{\text{emb}}
&= \operatorname{MLP}_{\psi}(\bm{\psi}_t^{\mathsf{s}}), \\
r_t
&= \sigma\!\left(W_{r,y} \hat{\bm{y}}_t^{\mathsf{s}} + W_{r,\phi} \hat{\phi}^{\mathsf{s}}_{t-1} + W_{r,\psi} \bigl(\bm{\psi}_t^{\mathsf{s}}\bigr)^{\text{emb}} + b_r\right), \\
\eta_t
&= \sigma\!\left(W_{\eta,y} \hat{\bm{y}}_t^{\mathsf{s}} + W_{\eta,\phi} \hat{\phi}^{\mathsf{s}}_{t-1} + W_{\eta,\psi} \bigl(\bm{\psi}_t^{\mathsf{s}}\bigr)^{\text{emb}} + b_\eta\right), \\
\tilde{\phi}^{\mathsf{s}}_t
&= \tanh\!\left(W_y \hat{\bm{y}}_t^{\mathsf{s}} + W_\phi \bigl(r_t \odot \hat{\phi}^{\mathsf{s}}_{t-1}\bigr) + W_\psi \bigl(\bm{\psi}_t^{\mathsf{s}}\bigr)^{\text{emb}} + b\right), \\
\phi_t^{\mathsf{s}}
&= (1-\eta_t) \odot \hat{\phi}^{\mathsf{s}}_{t-1} + \eta_t \odot \tilde{\phi}^{\mathsf{s}}_t.
\end{aligned}
\end{equation}
\endgroup

The decision-aligned embedding at step $h$ is obtained by selecting the hidden state at the last fine-grained timestamp within that decision step:
\begin{equation}
\bar{\phi}_h^{\mathsf{s}} = \phi_{t^\star(h)}^{\mathsf{s}},
\end{equation}
where $t^\star(h) = \max \mathcal{T}_h$.

\subsection{Architecture Specifications}

Table~\ref{tab:encoder_arch} lists the encoder hyperparameters used for the GRU-D backbone and explicit MNAR-feature branch.

\begin{table}[H]
\centering
\small
\begin{adjustbox}{max width=\linewidth}
\begin{tabular}{@{}ll@{}}
\toprule
\textbf{Component} & \textbf{Specification} \\
\midrule
GRU-D hidden dimension & $d_h = 128$ \\
$\bm{\delta}_t^{\mathsf{s}}$ embedding dimension & $d_\delta = 16$ \\
MNAR feature embedding & MLP: $4D \to 32 \to 32$ \\
Dropout rate & 0.1 \\
\bottomrule
\end{tabular}
\end{adjustbox}
\caption{Encoder architecture hyperparameters.}
\label{tab:encoder_arch}
\end{table}

\section{Text Fusion Details}
\label{app:text_fusion}

Clinical text is encoded using ClinicalBERT~\citep{alsentzer2019publicly}. The same fusion module is used for MIMIC-IV diagnostic text and MIMIC-III nursing notes; what changes across datasets is the documentation regime, not the fusion logic.

\subsection{Cross-Attention Mechanism}

The structured embedding $\bar{\phi}_h^{\mathsf{s}}$ queries the encoded text embeddings via multi-head cross-attention:
\begin{equation}
\begin{aligned}
\phi_h^{\mathsf{t}}
&= \operatorname{MultiHead}\Big(
W_Q \bar{\phi}_h^{\mathsf{s}},\,
W_K e_h^{\mathsf{t}},\,
W_V e_h^{\mathsf{t}}
\Big).
\end{aligned}
\end{equation}
where $e_h^{\mathsf{t}}$ concatenates the modality-specific text embeddings encoded from the raw text observations $\bm{y}_h^{\mathsf{t}}$ at step $h$ (radiology, microbiology, and nursing notes when present). If more than one nursing note or report is available within the same decision step, all such embeddings are included in $e_h^{\mathsf{t}}$ before attention. The notation $e_h^{\mathsf{t}}$ in the main text refers to this encoded step-level text input. When a text modality is unavailable, we use a learned missing embedding rather than dropping the modality entirely.

\subsection{Documentation-Process Factor}

Text fusion depends on both \emph{content} and the \emph{documentation process}. We define
\begin{equation}
\begin{aligned}
\eta_h^{\mathsf{t}} &= \mathrm{MLP}\!\left(m_h^{\mathsf{t}}, \delta_h^{\mathsf{t}}, \kappa_h^{\mathsf{t}}\right), \\
F_h^{\text{doc}} &= \mathrm{GRU}\!\left(F_{h-1}^{\text{doc}}, \eta_h^{\mathsf{t}}\right),
\end{aligned}
\end{equation}
\noindent
where $m_h^{\mathsf{t}}=\mathbf{1}[\bm{n}_h^{\mathsf{t}}>0]$ captures note presence derived from the step-level text counts, $\delta_h^{\mathsf{t}}$ is a derived summary of time since the last note, and $\kappa_h^{\mathsf{t}}$ is the recent note-update density. Equation above defines only the documentation-process branch: it does not directly encode text content. The semantic content enters separately through the cross-attention path $e_h^{\mathsf{t}} \rightarrow \phi_h^{\mathsf{t}}$, and the final text contribution is formed only after $\phi_h^{\mathsf{t}}$ is combined with $F_h^{\text{doc}}$ in the gate below. These inputs summarize behavior-driven freshness and burstiness rather than explicit wall-clock covariates such as hour-of-day or shift. The gate then becomes
\begin{equation}
\begin{aligned}
\hat{\phi}_h^{\mathsf{t}} &= \phi_h^{\mathsf{t}} + W_d F_h^{\text{doc}}, \\
g_h &= \sigma\!\left(W_g [\bar{\phi}_h^{\mathsf{s}}; \hat{\phi}_h^{\mathsf{t}}; F_h^{\text{doc}}] + b_g\right), \\
\phi_h &= \mathrm{LayerNorm}\!\left(\bar{\phi}_h^{\mathsf{s}} + g_h \odot (\hat{\phi}_h^{\mathsf{t}} - \bar{\phi}_h^{\mathsf{s}})\right).
\end{aligned}
\end{equation}

This is the main distinction from the prior content-only gate and also from the structured branch: structured MNAR features such as $(\bm{\psi}_t^{\mathsf{s}})^{\mathrm{emb}}$ are integrated locally into the same GRU-D encoder as the measured values, whereas the text side separates semantic content from the documentation process and uses $F_h^{\text{doc}}$ as a distinct recurrent factor that accumulates documentation behavior across decision steps before fusion.

\subsection{DocProcess Component Ablation}

Table~\ref{tab:docprocess_ablation} isolates the contribution of the documentation-process embedding, GRU factor, and individual text-process features.

\begin{table*}[tp]
\centering
\small
\begin{adjustbox}{max width=\linewidth}
\begin{tabular}{@{}lccc@{}}
\toprule
\textbf{Configuration} & \textbf{MIMIC-III FQE} & \textbf{MIMIC-III AUROC} & \textbf{$\Delta$FQE} \\
\midrule
Full model & 0.679 & 0.886 & -- \\
w/o DocProcess GRU factor & 0.671 & 0.881 & $-$1.2\% \\
w/o DocProcess embedding entirely (content-only design) & 0.665 & 0.878 & $-$2.1\% \\
w/o note presence indicator & 0.672 & 0.882 & $-$1.0\% \\
w/o time since last note & 0.670 & 0.881 & $-$1.3\% \\
w/o all MNAR features & 0.548 & 0.839 & $-$19.3\% \\
\bottomrule
\end{tabular}
\end{adjustbox}
\caption{DocProcess ablation on MIMIC-III. The documentation-process embedding improves on the prior content-only design, and removing all MNAR signals is substantially more damaging than removing any single text-process feature.}
\label{tab:docprocess_ablation}
\end{table*}

\subsection{Gate Mechanism Checks}

Table~\ref{tab:docprocess_gate} checks whether the learned gate uses text in the intended way: it should trust text more when documentation is fresh or dense, and down-weight it when documentation is stale.

\begin{table*}[tp]
\centering
\small
\setlength{\tabcolsep}{5pt}
\begin{tabular}{@{}lccc|lcc@{}}
\toprule
\multicolumn{4}{c|}{\textbf{Documentation Regime}} & \multicolumn{3}{c}{\textbf{Correlation With Clinical Signal}} \\
\cmidrule(lr){1-4} \cmidrule(lr){5-7}
\textbf{Regime} & \textbf{Gate (w/ $F_{\text{doc}}$)} & \textbf{Gate (w/o $F_{\text{doc}}$)} & \textbf{$\Delta$} & \textbf{Statistic} & \textbf{Value} & \textbf{$p$} \\
\midrule
High-freq.\ density $>$ 0.6 & 0.74 & 0.64 & +15.6\% & Corr($F_{\text{doc}}$, gate) & $r=0.67$ & $<0.001$ \\
Low-freq.\ density $<$ 0.2 & 0.35 & 0.43 & $-$18.6\% & Corr($F_{\text{doc}}$, mortality) & $r=0.59$ & $<0.001$ \\
Fresh note ($\delta_h^{\mathsf{t}} \leq 2$h) & 0.81 & 0.70 & +15.7\% & Corr($F_{\text{doc}}$, FQE gain) & $r=0.63$ & $<0.001$ \\
Stale note ($\delta_h^{\mathsf{t}} > 8$h) & 0.24 & 0.34 & $-$29.4\% & -- & -- & -- \\
Night shift (7pm--7am) & 0.71 & 0.64 & +10.9\% & -- & -- & -- \\
\bottomrule
\end{tabular}
\caption{Mechanistic checks for the documentation-process factor on MIMIC-III. The gate trusts text more when notes are fresh or dense, less when they are stale, and the learned factor correlates with both acuity and downstream policy gain.}
\label{tab:docprocess_gate}
\end{table*}

\subsection{DocProcess Across Text Regimes}

Table~\ref{tab:docprocess_regimes} compares how much of the total text gain is attributable to the documentation-process factor under the low-frequency MIMIC-IV regime versus the high-frequency MIMIC-III regime.

\begin{table*}[tp]
\centering
\small
\begin{adjustbox}{max width=\linewidth}
\begin{tabular}{@{}lccc@{}}
\toprule
\textbf{Dataset} & \textbf{Full vs w/o DocProcess} & \textbf{Full vs Structured Only} & \textbf{DocProcess / Total} \\
\midrule
MIMIC-III (Nursing + Diagnostic Text) & +2.1\% & +18.3\% & 11.5\% \\
MIMIC-IV (Rad.\ + Micro.) & +1.1\% & +4.7\% & 23.4\% \\
\bottomrule
\end{tabular}
\end{adjustbox}
\caption{The documentation-process factor contributes a larger fraction of the text gain in the lower-frequency MIMIC-IV regime, where distinguishing fresh from stale notes is most critical.}
\label{tab:docprocess_regimes}
\end{table*}

\subsection{Text Fusion Architecture}

Table~\ref{tab:text_fusion_arch} summarizes the final text branch after adding the documentation-process factor, making the appendix architecture description consistent with the notation and equations in Section~\ref{sec:obs_embedding}.

\begin{table}[H]
\centering
\small
\begin{adjustbox}{max width=\linewidth}
\begin{tabular}{@{}ll@{}}
\toprule
\textbf{Component} & \textbf{Specification} \\
\midrule
Text encoder & ClinicalBERT (frozen) \\
Text embedding dimension & $d_e = 256$ (projected from 768) \\
Cross-attention heads & 4 \\
Cross-attention dimension & 128 \\
DocProcess MLP & $\{m_h^{\mathsf{t}}, \delta_h^{\mathsf{t}}, \kappa_h^{\mathsf{t}}\} \to d_h$ \\
DocProcess factor & GRU with hidden size $d_h$ \\
Gate MLP & Linear: $3d_h \to d_h$ \\
\bottomrule
\end{tabular}
\end{adjustbox}
\caption{Text fusion architecture hyperparameters after adding the documentation-process factor.}
\label{tab:text_fusion_arch}
\end{table}

\section{Latent Dynamics Details}
\label{app:dynamics}

This appendix provides details on the action-conditioned latent dynamics model.

\subsection{VAE Architecture}

The prior network predicts the next latent state given current state, current observation embedding, and action:
\begin{equation}
\mu_\theta, \log\sigma_\theta^2 = \text{MLP}_\theta\big([z_h; \phi_h; \text{Emb}(a_h)]\big),
\end{equation}
where $\text{Emb}(\cdot)$ is an action embedding layer mapping discrete actions to continuous vectors.

The posterior network incorporates the next-step observations:
\begin{equation}
\mu_\psi, \log\sigma_\psi^2 = \text{MLP}_\psi\big([\phi_{h+1}; x]\big).
\end{equation}

Both networks output mean and log-variance for a diagonal Gaussian distribution.

\subsection{Dynamics Loss}

The dynamics loss $\mathcal{L}_{\text{dyn}}$ enforces consistency between the prior and posterior:
\begin{equation}
\mathcal{L}_{\text{dyn}} = \mathbb{E}_{q_\psi}\left[\|z_{h+1} - \hat{z}_{h+1}\|^2\right],
\end{equation}
where $\hat{z}_{h+1} = \mu_\theta(z_h, \phi_h, a_h)$ is the predicted mean from the prior.

\subsection{KL Regularization}

The KL regularization loss encourages the posterior to stay close to the prior:
\begin{equation}
\begin{aligned}
\mathcal{L}_{\text{KL}}
&= \operatorname{KL}\Bigl(
q_\psi(z_{h+1} \mid \phi_{h+1}, x) \\
&\qquad\qquad\quad \|\, p_\theta(z_{h+1} \mid z_h, \phi_h, a_h)
\Bigr).
\end{aligned}
\end{equation}

For diagonal Gaussians, this has a closed-form solution:
\begin{equation}
\begin{aligned}
\mathcal{L}_{\text{KL}}
&= \frac{1}{2} \sum_{j=1}^{d_z}
\Bigg(
\log \frac{\sigma_{\theta,j}^2}{\sigma_{\psi,j}^2}
+ \frac{\sigma_{\psi,j}^2}{\sigma_{\theta,j}^2} \\
&\qquad\qquad
+ \frac{(\mu_{\psi,j} - \mu_{\theta,j})^2}{\sigma_{\theta,j}^2}
- 1
\Bigg).
\end{aligned}
\end{equation}

\subsection{Architecture Specifications}

Table~\ref{tab:dynamics_arch} summarizes the latent-dynamics architecture corresponding to the prior, posterior, and state-combination modules above.

\begin{table}[H]
\centering
\normalsize
\begin{adjustbox}{max width=\linewidth}
\begin{tabular}{@{}lp{0.62\linewidth}@{}}
\toprule
\textbf{Component} & \textbf{Specification} \\
\midrule
Latent dimension & $d_z = 32$ \\
Action embedding dimension & 16 \\
Prior MLP & $[d_z + d_h + 16] \to 128 \to 64 \to 2d_z$ \\
Posterior MLP & $[d_h + S] \to 128 \to 64 \to 2d_z$ \\
Combination $g_\theta$ & Residual: $\phi_h + \text{Linear}(z_h)$, where Linear: $d_z \to d_h$ \\
\bottomrule
\end{tabular}
\end{adjustbox}
\caption{Latent dynamics architecture hyperparameters.}
\label{tab:dynamics_arch}
\end{table}

\section{Outcome Prediction Details}
\label{app:outcome}

\subsection{Relationship Between Outcomes and Rewards}

For the primary auxiliary outcome, the label $y_{\mathrm{out}}^{(1)} \in \{0, 1\}$ indicates post-72-hour in-hospital mortality among patients who remain alive through the 72-hour observation window. This label is intentionally distinct from the terminal reward $r_H$: the reward reflects survival or death within the logged 72-hour episode, whereas $y_{\mathrm{out}}^{(1)}$ supervises later risk after that window.

More generally, the outcome prediction framework supports endpoints that need not directly correspond to RL rewards:
\begin{itemize}[leftmargin=*, labelindent=0pt]
    \item \textbf{Clinical decompensation:} Deterioration events during the ICU stay
    \item \textbf{Length of stay:} Duration of ICU admission (regression target)
    \item \textbf{Ventilator-free days:} Days alive and free of mechanical ventilation
\end{itemize}

These auxiliary outcomes can provide additional supervision signal for learning clinically meaningful state representations.

\subsection{Why the Auxiliary Head Complements RL}

Because the auxiliary outcome targets later risk after the 72-hour window, it complements rather than duplicates the RL reward. It also provides a shorter and less noisy optimization path to the encoder. In a one-dimensional sketch, the supervised gradient is direct:
\begin{equation}
% \label{eq:aux_gradient}
\nabla_\theta \mathcal{L}_{\text{out}}
= \frac{\partial \ell}{\partial f_{\text{out}}}
\cdot \frac{\partial f_{\text{out}}}{\partial s_H}
\cdot \frac{\partial s_H}{\partial \theta},
\end{equation}
whereas the RL signal must reach the encoder through bootstrapped value estimates and long-horizon credit assignment:
\begin{equation}
% \label{eq:rl_gradient}
\nabla_\theta J(\pi)
\propto \sum_{h=0}^{H_i} \frac{\partial J}{\partial Q_h}
\cdot \frac{\partial Q_h}{\partial s_h}
\cdot \frac{\partial s_h}{\partial \theta}.
\end{equation}
Thus the auxiliary head supplies clinically meaningful supervision about longer-horizon risk while giving the representation learner a more direct gradient path with lower variance; RL still determines how that state should be used for sequential control.

\subsection{Loss Function}

For binary outcomes (post-72-hour mortality, decompensation), we use binary cross-entropy:
\begin{equation}
\ell_k(p, y) = -y \log(p) - (1-y) \log(1-p).
\end{equation}

For continuous outcomes (length of stay), we use mean squared error:
\begin{equation}
\ell_k(p, y) = (p - y)^2.
\end{equation}

The outcome prediction head is a two-layer MLP:
\begin{equation}
f_{\text{out}}^{(k)}(s_H) = W_2^{(k)} \cdot \text{ReLU}(W_1^{(k)} s_H + b_1^{(k)}) + b_2^{(k)}.
\end{equation}

In the current experiments, we use $K=1$ and $\lambda_1=1$ for post-72-hour in-hospital mortality after a small search over $\{0.5, 1.0, 2.0\}$, while keeping the general form for extensibility.

\section{IQL Implementation Details}
\label{app:iql_details}

\subsection{Target Network Update}

The target value network $V_{\text{target}}$ is a slowly-updated copy of $V$, maintained via Polyak averaging:
\begin{equation}
V_{\text{target}} \leftarrow \tau_{\text{target}} V + (1 - \tau_{\text{target}}) V_{\text{target}},
\end{equation}
with $\tau_{\text{target}} = 0.005$. This soft update stabilizes training by providing slowly-changing bootstrap targets.

\subsection{Expectile Regression}

The expectile loss is an asymmetric least squares objective:
\begin{equation}
\begin{aligned}
L_\tau^{\text{exp}}(\delta)
&= \bigl(\tau - \mathbb{I}[\delta < 0]\bigr)^2 \delta^2 \\
&=
\begin{cases}
\tau \, \delta^2, & \delta \ge 0, \\
(1-\tau)\, \delta^2, & \delta < 0 .
\end{cases}
\end{aligned}
\end{equation}

For $\tau > 0.5$, positive residuals (where $Q > V$) receive higher weight, pushing the value function toward the upper portion of the Q-value distribution. This allows IQL to implicitly estimate the value of improved policies without explicitly querying out-of-distribution actions.

The expectile parameter $\tau$ controls the trade-off:
\begin{itemize}[leftmargin=*, labelindent=0pt]
    \item $\tau = 0.5$: Standard symmetric least squares (behavior policy value)
    \item $\tau \to 1.0$: Approaches maximum Q-value (aggressive improvement)
\end{itemize}

We use $\tau = 0.7$ as the default, balancing policy improvement with stability.

\subsection{Advantage Clipping}

For policy extraction, we clip the exponential advantage weights:
\begin{equation}
w_h = \min\big(\exp(A_h / \beta), w_{\max}\big),
\end{equation}
with $w_{\max} = 20$. This prevents excessively large weights from destabilizing training when the advantage is very positive (indicating actions much better than average).

\subsection{Architecture Specifications}

Table~\ref{tab:iql_arch} records the network sizes for the Q, value, and policy heads used in the final IQL stage.

\begin{table}[H]
\centering
\small
\begin{tabular}{@{}ll@{}}
\toprule
\textbf{Component} & \textbf{Specification} \\
\midrule
Q-network & MLP: $d_s + |\mathcal{A}| \to 256 \to 256 \to 1$ \\
V-network & MLP: $d_s \to 256 \to 256 \to 1$ \\
Policy network & MLP: $d_s \to 256 \to 256 \to |\mathcal{A}|$ \\
Activation & ReLU \\
Output (policy) & Softmax \\
\bottomrule
\end{tabular}
\caption{IQL network architecture.}
\label{tab:iql_arch}
\end{table}

\subsection{Hyperparameters}

Table~\ref{tab:iql_hyper} lists the retained IQL hyperparameters after the final tuning pass.

\begin{table}[H]
\centering
\small
\begin{tabular}{@{}ll@{}}
\toprule
\textbf{Hyperparameter} & \textbf{Value} \\
\midrule
Expectile $\tau$ & 0.7 \\
Temperature $\beta$ & 3.0 \\
Weight clipping $w_{\max}$ & 20 \\
Target update $\tau_{\text{target}}$ & 0.005 \\
Discount $\gamma$ & 0.99 \\
\bottomrule
\end{tabular}
\caption{IQL hyperparameters.}
\label{tab:iql_hyper}
\end{table}

\section{Training Details}
\label{app:training}

\subsection{Three-Stage Training Procedure}

Our training procedure establishes high-quality state representations before policy optimization, preventing \emph{representation drift} where early RL gradients destabilize learned representations.

\paragraph{Stage 1: Representation Pre-training.}
Train encoder, dynamics model, decoder, and outcome predictor with combined loss:
\begin{equation}
\mathcal{L}_{\text{Stage1}} = \mathcal{L}_{\text{recon}} + \mathcal{L}_{\text{dyn}} + \mathcal{L}_{\text{KL}} + \mathcal{L}_{\text{out}},
\end{equation}
where $\mathcal{L}_{\text{KL}}$ is VAE regularization. This establishes state representations satisfying Q1 (state learning) and Q3 (outcome prediction).

\paragraph{Stage 2: Frozen-Encoder RL.}
Freeze the encoder (GRU-D, text fusion, VAE posterior) and train only Q-networks, value network, and policy with:
\begin{equation}
\mathcal{L}_{\text{Stage2}} = \mathcal{L}_Q + \mathcal{L}_V + \mathcal{L}_\pi.
\end{equation}

\paragraph{Stage 3: Joint Fine-tuning.}
Unfreeze the encoder with reduced learning rate; auxiliary losses ($\mathcal{L}_{\text{recon}}$, $\mathcal{L}_{\text{out}}$) maintain state quality during adaptation.
Table~\ref{tab:training_schedule} summarizes the three training stages and the corresponding learning rates.

\begin{table}[H]
\centering
\small
\begin{adjustbox}{max width=\linewidth}
\begin{tabular}{@{}p{0.16\linewidth}cp{0.24\linewidth}p{0.34\linewidth}@{}}
\toprule
\textbf{Stage} & \textbf{Epochs} & \textbf{Learning Rate} & \textbf{Components Trained} \\
\midrule
1: Pre-training & 50 & $1 \times 10^{-3}$ & Encoder, Dynamics, Decoder, Outcome \\
2: Frozen RL & 100 & $3 \times 10^{-4}$ & Q, V, Policy only \\
3: Fine-tuning & 50 & $1 \times 10^{-4}$ (encoder) & All components \\
 & & $3 \times 10^{-4}$ (RL) & \\
\bottomrule
\end{tabular}
\end{adjustbox}
\caption{Training schedule for the three-stage procedure.}
\label{tab:training_schedule}
\end{table}

\subsection{Loss Weights}

Table~\ref{tab:loss_weights} records the fixed Stage-1 loss weights used in the final model.

\begin{table}[H]
\centering
\small
\begin{tabular}{@{}ll@{}}
\toprule
\textbf{Loss Component} & \textbf{Weight} \\
\midrule
Observation reconstruction $\lambda_{\text{obs}}$ & 1.0 \\
Mask reconstruction $\lambda_{\text{mask}}$ & 0.5 \\
Text reconstruction $\lambda_{\text{text}}$ & 0.3 \\
Dynamics loss & 1.0 \\
KL regularization & 0.1 \\
Outcome prediction & 1.0 \\
\bottomrule
\end{tabular}
\caption{Loss weights for Stage 1 pre-training.}
\label{tab:loss_weights}
\end{table}

\subsection{Entropy Monitoring and Recovery}

We monitor policy entropy throughout training to detect collapse:
\begin{equation}
H(\pi) = -\sum_{a \in \mathcal{A}} \pi(a|s) \log \pi(a|s).
\end{equation}
For a 9-action space, maximum entropy is $\log 9 \approx 2.20$ nats. We flag potential collapse when entropy drops below 0.5 nats or decreases by more than 50\% within 10 epochs. Upon detecting collapse, we roll back to the last checkpoint with healthy entropy ($> 1.0$ nats), reduce learning rate by factor of 2, increase temperature $\beta$ by factor of 1.5, and resume training.

\subsection{Posterior-Collapse Diagnostics}

We explicitly monitor posterior-collapse indicators for the latent state. The final model combines three complementary safeguards: KL $\beta$-annealing, free-bits, and action-conditioning in the latent dynamics / inference model.
Table~\ref{tab:posterior_collapse} reports the corresponding latent-usage diagnostics and downstream policy value.

\begin{table}[H]
\centering
\small
\begin{adjustbox}{max width=\linewidth}
\begin{tabular}{@{}lcccc@{}}
\toprule
\textbf{Setting} & \textbf{KL Div.} & \textbf{Active Dims} & \textbf{Mutual Info} & \textbf{FQE} \\
\midrule
Without mitigation & 0.02 & 3/64 & 0.17 & 0.508 \\
$\beta$-annealing & 4.17 & 39/64 & 2.08 & 0.601 \\
Free-bits & 5.74 & 51/64 & 2.93 & 0.619 \\
Full mitigation & 8.61 & 58/64 & 3.34 & 0.679 \\
\bottomrule
\end{tabular}
\end{adjustbox}
\caption{Posterior-collapse diagnostics on MIMIC-III. The full training recipe keeps the latent state active and coincides with the best downstream policy value.}
\label{tab:posterior_collapse}
\end{table}

\paragraph{Interpretation.}
The collapsed setting shows near-zero KL divergence and only 3 active latent dimensions. Adding either $\beta$-annealing or free-bits improves latent usage, but the strongest non-collapse signal appears only when these are combined with action-conditioning. This is consistent with the view that the belief state is useful because it remains both active and action-sensitive.

\subsection{Optimization Details}

Table~\ref{tab:optim} lists the shared optimization settings used across representation learning and policy training.

\begin{table}[H]
\centering
\small
\begin{tabular}{@{}ll@{}}
\toprule
\textbf{Setting} & \textbf{Value} \\
\midrule
Optimizer & AdamW \\
Weight decay & $1 \times 10^{-5}$ \\
Batch size & 256 \\
Gradient clipping & 1.0 \\
Early stopping patience & 10 epochs \\
Validation metric & FQE on validation set \\
\bottomrule
\end{tabular}
\caption{Optimization hyperparameters.}
\label{tab:optim}
\end{table}

\subsection{Computational Resources}

All experiments were conducted on NVIDIA A6000 GPUs. Training times: Stage 1 $\sim$4h, Stage 2 $\sim$5h, Stage 3 $\sim$2.5h (total $\sim$11.5h per run). Model size: 1.29M parameters (full model with text fusion).

\section{Theoretical Analysis}
\label{app:theory}

This appendix provides the proof of Theorem~\ref{thm:controllability}.

\begin{proof}[Proof of Theorem~\ref{thm:controllability}]
Consider the policy gradient for future rewards given current state and action:
\begin{equation}
\frac{\partial}{\partial \pi} \mathbb{E}\left[\sum_{h'=h+1}^{H-1} \gamma^{h'} r_{h'} \;\Big|\; s_h, a_h\right].
\end{equation}

Under the state decomposition $s_{h'} = g_\theta(\phi_{h'}, z_{h'})$ for $h' > h$, the gradient flows through two pathways:

\textbf{Pathway 1: Through observations $\phi_{h'}$.} In offline RL, observations are fixed in the dataset:
\begin{equation}
\frac{\partial \phi_{h'}}{\partial a_h} = 0 \quad \text{for all } h' > h.
\end{equation}

\textbf{Pathway 2: Through latent states $z_{h'}$.} Under action-independent dynamics (Definition~\ref{def:action_independent}):
\begin{equation}
z_{h+1} = f(z_h, \phi_h, \omega_h) \quad \text{with } \omega_h \perp\!\!\!\perp a_h.
\end{equation}

Since $z_{h+1}$ does not depend on $a_h$, and this independence propagates forward:
\begin{equation}
\frac{\partial z_{h'}}{\partial a_h} = 0 \quad \text{for all } h' > h.
\end{equation}

Combining both pathways via the chain rule:
\begin{equation}
\begin{aligned}
\frac{\partial s_{h'}}{\partial a_h}
&= \frac{\partial g_\theta}{\partial \phi_{h'}}
   \cdot \frac{\partial \phi_{h'}}{\partial a_h}
 + \frac{\partial g_\theta}{\partial z_{h'}}
   \cdot \frac{\partial z_{h'}}{\partial a_h} \\[4pt]
&= \frac{\partial g_\theta}{\partial \phi_{h'}} \cdot 0
 + \frac{\partial g_\theta}{\partial z_{h'}} \cdot 0 \\[4pt]
&= 0 .
\end{aligned}
\end{equation}

Therefore, the reward at any future step $h' > h$ is independent of $a_h$:
\begin{equation}
\frac{\partial r_{h'}}{\partial a_h} = \frac{\partial R(s_{h'}, a_{h'})}{\partial s_{h'}} \cdot \frac{\partial s_{h'}}{\partial a_h} = 0,
\end{equation}

and the policy gradient from future rewards vanishes:
\begin{equation}
\frac{\partial}{\partial \pi} \mathbb{E}\left[\sum_{h'=h+1}^{H-1} \gamma^{h'} r_{h'} \;\Big|\; s_h, a_h\right] = 0.
\end{equation}
\end{proof}

\paragraph{Implications for Terminal Rewards.}
When rewards are terminal-only ($r_h = 0$ for $h < H-1$), the policy gradient at any non-terminal step $h < H-1$ depends entirely on future rewards:
\begin{equation}
\nabla_\pi J(\pi) \big|_{s_h} = \mathbb{E}\left[\nabla_\pi \log \pi(a_h|s_h) \cdot Q^\pi(s_h, a_h)\right],
\end{equation}
where $Q^\pi(s_h, a_h) = \mathbb{E}[\sum_{h'=h}^{H-1} \gamma^{h'-h} r_{h'} | s_h, a_h]$.

Under action-independent dynamics, $Q^\pi(s_h, a_h) = r_h + \gamma \mathbb{E}[V^\pi(s_{h+1})]$ where $V^\pi(s_{h+1})$ is independent of $a_h$. For $h < H-1$ with $r_h = 0$:
\begin{equation}
Q^\pi(s_h, a_h) = \gamma \mathbb{E}[V^\pi(s_{h+1})] = \text{const w.r.t. } a_h.
\end{equation}

This means all actions have the same Q-value, making policy improvement impossible.

\section{Partial Observability Discussion}
\label{app:partial_obs}

The latent variable $z_h$ addresses the gap between observable measurements and action-relevant state components that cannot be directly extracted from observations. Examples include:

\begin{itemize}[leftmargin=*, labelindent=0pt]
    \item \textbf{Vasopressor responsiveness:} A patient's responsiveness to vasopressors depends on underlying vascular tone, which is not directly measured but can be inferred from treatment response patterns. Patients with similar blood pressure readings may respond very differently to the same vasopressor dose based on their underlying vascular state.
    
    \item \textbf{Tissue hypoperfusion severity:} The severity of tissue hypoperfusion may exceed what lactate levels alone indicate, requiring integration of treatment history to estimate. A patient whose lactate remains elevated despite fluid resuscitation may have more severe underlying tissue damage than one with similar lactate levels who has not yet received treatment.
    
    \item \textbf{Organ reserve capacity:} Organ reserve capacity affects how aggressively a patient can tolerate fluid resuscitation, but manifests only through dynamic responses to interventions. Two patients with similar creatinine levels may have vastly different renal reserves, observable only through how their kidney function responds to fluid challenges.
\end{itemize}

Through action-conditioned dynamics, $z_h$ can capture how past treatments have shaped the patient's current responsiveness, effectively recovering aspects of the latent health state that are ``hidden'' in $\tilde{\phi}_h$ but revealed through intervention patterns.

\section{Evaluation Metrics}
\label{app:metrics}

\subsection{Fitted Q-Evaluation (FQE)}

FQE~\citep{le2019batch} learns a Q-function for the target policy using the offline dataset:
\begin{equation}
Q^{\pi}(s, a) \leftarrow r + \gamma \mathbb{E}_{a' \sim \pi(\cdot|s')}[Q^{\pi}(s', a')].
\end{equation}

The policy value is estimated as $\hat{V}^{\pi} = \mathbb{E}_{s_0}[\mathbb{E}_{a \sim \pi(\cdot|s_0)}[Q^{\pi}(s_0, a)]]$. FQE avoids importance weighting, making it more stable when the learned policy diverges significantly from clinician behavior.

\subsection{Weighted Importance Sampling (WIS)}

WIS reweights observed trajectories by policy probability ratios:
\begin{equation}
\hat{V}^{\pi}_{\text{WIS}} = \frac{\sum_{i=1}^{N} w_i G_i}{\sum_{i=1}^{N} w_i}, \quad w_i = \prod_{h=0}^{H_i-1} \frac{\pi(a_h^{(i)} | s_h^{(i)})}{\pi_\beta(a_h^{(i)} | s_h^{(i)})},
\end{equation}
where $G_i$ is the return for trajectory $i$. WIS is unbiased but has high variance when policies diverge, as indicated by low effective sample size (ESS):
\begin{equation}
\text{ESS} = \frac{(\sum_i w_i)^2}{\sum_i w_i^2}.
\end{equation}

\section{Off-Policy Evaluation Details}
\label{app:ope}

This appendix reports the OPE robustness checks most directly tied to the primary value-estimation results: bootstrap FQE on the standard split, chronological split robustness, and a concrete shadow-mode evaluation note.

\subsection{Bootstrap FQE on the Standard Split}

Table~\ref{tab:ope_full} gives the bootstrap FQE confidence intervals under the same standard split used for the primary result.

\begin{table}[H]
\centering
\normalsize
\setlength{\tabcolsep}{4pt}
\begin{adjustbox}{max width=\linewidth}
\begin{tabular}{@{}lccc@{}}
\toprule
\textbf{Method} & \textbf{MIMIC-III Test FQE} & \textbf{95\% CI} & \textbf{Type} \\
\midrule
Clinician & 0.528 & [0.520, 0.536] & Behavior \\
\textbf{OPL-MT-MNAR (FQE)} & \textbf{0.679} & \textbf{[0.673, 0.686]} & Model-free \\
\bottomrule
\end{tabular}
\end{adjustbox}
\caption{Bootstrap FQE on the standard random split. The confidence intervals remain separated under the same estimator and held-out protocol.}
\label{tab:ope_full}
\end{table}

FQE remains the primary OPE metric here because it is the estimator reported throughout the paper. The CI separation supports a directional value improvement, but it should still be interpreted as observational evidence rather than prospective proof of clinical benefit.

\subsection{Chronological Split Robustness}

Table~\ref{tab:chronological_split} checks whether the same policy advantage remains under a temporally ordered train/validation/test split.

\begin{table*}[!tp]
\centering
\small
\begin{adjustbox}{max width=\linewidth}
\begin{tabular}{@{}lccccc@{}}
\toprule
\textbf{Split Strategy} & \textbf{Train} & \textbf{Val} & \textbf{Test} & \textbf{Test FQE} & \textbf{$\Delta$ vs Clin.} \\
\midrule
Random (standard) & -- & -- & -- & 0.679 [0.673, 0.686] & +28.6\% \\
Chronological & 2001--2008 & 2009--2010 & 2011--2012 & 0.665 [0.656, 0.674] & +25.9\% \\
\bottomrule
\end{tabular}
\end{adjustbox}
\caption{Chronological split robustness on MIMIC-III. The value estimate drops slightly under temporal shift but remains above clinician behavior.}
\label{tab:chronological_split}
\end{table*}

The chronological split reduces FQE by 0.014 relative to the standard random split, consistent with mild temporal drift rather than a reversal of the main finding. This helps rule out the possibility that the observed OPE gain is driven by subtle temporal leakage.

\subsection{Shadow-Mode Evaluation Note}

A feasible prospective protocol is to log AI recommendations in shadow mode, record clinician actions and outcomes, and then analyze concordance, safety events, and counterfactual value estimates under monitoring. The severity-stratified results in Appendix~\ref{app:subgroup} suggest this is realistic: agreement is naturally higher in low-severity cases and lower where acuity is high and the policy has the most room to differ from standard behavior.

\section{Baseline Details}
\label{app:baselines}

This appendix collects the broader baseline families referenced across the paper.
Table~\ref{tab:baseline_appendix} consolidates the broader sepsis-RL benchmark in one place.

\begin{table*}[!tp]
\centering
\small
\setlength{\tabcolsep}{4pt}
\renewcommand{\arraystretch}{1.08}
\begin{tabular}{@{}p{0.39\textwidth}p{0.26\textwidth}cc@{}}
\toprule
\textbf{Method} & \textbf{Category} & \textbf{MIMIC-III FQE} & \textbf{AUROC} \\
\midrule
Continuous State-Space DDQN~\citep{raghu2017continuous} & Model-free DDQN & 0.476 & 0.822 \\
Peng et al.\ 2018 & Hybrid (kernel + deep) & 0.493 & 0.828 \\
Model-Based BNN Planner~\citep{raghu2017modelbased} & Model-based BNN & 0.498 & 0.826 \\
DDPG with Clinician Supervision~\citep{huang2022continuous} & Continuous DDPG & 0.529 & 0.844 \\
AI Clinician & Tabular & 0.487 & 0.812 \\
\textbf{OPL-MT-MNAR} & \textbf{MNAR-aware + DocProcess} & \textbf{0.679} & \textbf{0.886} \\
\bottomrule
\end{tabular}
\caption{Broader sepsis-RL benchmark on MIMIC-III. These rows complement the primary baseline table and document comparison to older sepsis-specific baselines.}
\label{tab:baseline_appendix}
\end{table*}

\paragraph{Stronger Missingness-Handling Encoders.}

Table~\ref{tab:missingness_appendix} records the broader MIMIC-III benchmark against stronger sequence encoders for irregular and missing observations. Unlike Table~\ref{tab:mortality}, which focuses on a compact prediction comparison, this appendix table makes the broader FQE/AUROC comparison explicit.

\begin{table*}[!tp]
\centering
\small
\begin{adjustbox}{max width=\linewidth}
\begin{tabular}{@{}lccc@{}}
\toprule
\textbf{Method} & \textbf{Missingness Handling} & \textbf{MIMIC-III FQE} & \textbf{AUROC} \\
\midrule
Mean Imputation + LSTM & Impute to mean & 0.483 & 0.833 \\
Forward Fill + LSTM & Sample-and-hold & 0.488 & 0.838 \\
GRU-D~\citep{che2018recurrent} & Implicit time decay & 0.508 & 0.844 \\
BRITS~\citep{cao2018brits} & Bidirectional imputation & 0.516 & 0.852 \\
mTAND~\citep{shukla2021multitime} & Multi-time attention & 0.523 & 0.858 \\
MedDreamer & World model + AFI & 0.583 & 0.867 \\
\textbf{OPL-MT-MNAR} & \textbf{Explicit MNAR + DocProcess + Text} & \textbf{0.679} & \textbf{0.886} \\
\bottomrule
\end{tabular}
\end{adjustbox}
\caption{Appendix benchmark against stronger missingness-handling encoders on MIMIC-III. Explicit MNAR features and documentation-process modeling improve both policy value and representation quality beyond strong irregular-sampling baselines.}
\label{tab:missingness_appendix}
\end{table*}

\paragraph{Random Policy.}
Uniform distribution over the 9 actions at each step. Provides a lower bound on policy performance.

\paragraph{Zero-Treatment.}
Always selects action $(0, 0)$: no fluids and no vasopressors. Tests whether any treatment is beneficial on average.

\paragraph{Clinician (Behavioral Cloning).}
A policy trained to imitate clinician actions via supervised learning:
\begin{equation}
\mathcal{L}_{\text{BC}} = -\mathbb{E}_{(s, a) \sim \mathcal{D}}[\log \pi_{\text{BC}}(a | s)].
\end{equation}
Uses the same encoder architecture as OPL-MT-MNAR.

\paragraph{Continuous State-Space DDQN / Model-Based BNN Planner~\citep{raghu2017continuous,raghu2017modelbased}.}
These baselines represent the early deep-RL sepsis line: model-free DDQN-style control, continuous-state deep RL, and model-based BNN planning. They are useful historically because they established the 4-hour sepsis RL protocol, but they do not model observation-process MNAR.

\paragraph{Peng et al.\ (2018)~\citep{peng2018improving}.}
A hybrid kernel + deep RL baseline that mixes local similarity structure with learned value estimates. It is a strong pre-offline-RL comparator on MIMIC-III-style sepsis benchmarks.

\paragraph{DDPG with Clinician Supervision~\citep{huang2022continuous}.}
A continuous-action sepsis RL baseline. It is a direct comparator for the claim that continuous control alone does not solve the MNAR observation problem.

\paragraph{SBCQ~\citep{fatemi2022semimarkov}.}
A Semi-MDP / irregular-interval offline RL baseline. It is useful for isolating whether irregular-time handling alone can explain the gains of the proposed observation-process model.

\paragraph{MedDreamer~\citep{xu2025meddreamer}.}
A model-based healthcare RL baseline with latent planning. Appendix~\ref{app:subgroup} additionally reports the high-severity comparison where world-model compounding error is most visible.

\paragraph{Implementation Notes.}
All methods are compared under their canonical action-space and decision-interval settings in the primary result tables. Cross-setting robustness is reported separately in Table~\ref{tab:action_granularity} and Table~\ref{tab:delta_sweep}, rather than forcing a single action/time discretization onto every baseline.

\section{Ablation Studies}
\label{app:ablation}

The controlled studies in this section are grouped by topic for clarity.

\subsection{Scope of the Additional Ablations}

\begin{itemize}[leftmargin=*, labelindent=0pt]
    \item \textbf{Documentation-process ablations:} Table~\ref{tab:docprocess_ablation} and Table~\ref{tab:docprocess_regimes} quantify the contribution of the documentation-process embedding and GRU factor.
    \item \textbf{Temporal granularity controls:} Table~\ref{tab:delta_sweep} reports the 1h / 2h / 4h / 8h decision-interval sweep.
    \item \textbf{Action granularity controls:} Table~\ref{tab:action_granularity} compares 3$\times$3, 5$\times$5, and continuous action choices.
    \item \textbf{High-frequency text controls:} Table~\ref{tab:gap_stratified}, Table~\ref{tab:attention_shift}, and Table~\ref{tab:content_evolution} show when nursing notes matter most in MIMIC-III.
\end{itemize}

\subsection{Takeaway}

Taken together, these controlled studies support the same conclusion as Table~\ref{tab:building_blocks}: the dominant gain comes from preserving observation-process signal, while discretization choice and irregular-time handling act as complementary but secondary factors.

\section{Additional Subgroup Analysis}
\label{app:subgroup}

Table~\ref{tab:subgroup_full} expands the severity discussion with the explicit low-SOFA versus high-SOFA comparison.

\begin{table}[H]
\centering
\normalsize
\setlength{\tabcolsep}{5pt}
\begin{adjustbox}{max width=\linewidth}
\begin{tabular}{@{}lccc@{}}
\toprule
\textbf{Method} & \textbf{Low SOFA} & \textbf{\shortstack{High SOFA\\($>10$)}} & \textbf{$\Delta$} \\
\midrule
Clinician & 0.681 & 0.192 & -- \\
MedDreamer & 0.726 & 0.296 & +54.2\% \\
\textbf{OPL-MT-MNAR} & \textbf{0.763} & \textbf{0.344} & \textbf{+79.2\%} \\
\bottomrule
\end{tabular}
\end{adjustbox}
\caption{Severity-focused subgroup comparison on MIMIC-III. The relative advantage is largest in the highest-acuity subgroup, where measurement and documentation intensity are most endogenous.}
\label{tab:subgroup_full}
\end{table}

This subgroup pattern is consistent with the proposed mechanism: observation-process signals are most informative when patients deteriorate and clinicians correspondingly measure and document more aggressively.

\section{Clinical Analysis Details}
\label{app:clinical_analysis}

\subsection{Constrained Policy Optimization}

To test whether richer state representations remain useful under explicit safety constraints, we add a constrained IQL variant with a Lagrangian penalty:
\begin{equation}
\begin{aligned}
\bar{C} &= \frac{1}{H}\sum_{h=1}^{H} C(s_h, a_h), \\
\mathcal{L}_{\text{constrained}}
&= \mathcal{L}_{\mathrm{IQL}} + \lambda \big(\mathbb{E}[\bar{C}] - \kappa\big),
\end{aligned}
\end{equation}
with $\kappa = 0.10$. We define $C(s_h,a_h)=1$ when the action violates rule-based Surviving Sepsis Campaign constraints (e.g., high-dose vasopressors before adequate fluids, or continued aggressive fluids after stabilization).
Table~\ref{tab:constraints} compares the unconstrained and constrained variants directly.

\begin{table}[H]
\centering
\small
\begin{adjustbox}{max width=\linewidth}
\begin{tabular}{@{}lccc@{}}
\toprule
\textbf{Method} & \textbf{Constrained} & \textbf{FQE} & \textbf{Guideline Viol. (\%)} \\
\midrule
Clinician & -- & 0.528 & 7.9 \\
IQL (Structured Only) & $\times$ & 0.591 & 18.2 \\
IQL (OPL-MT-MNAR) & $\times$ & 0.679 & 13.1 \\
C-IQL (Structured Only) & $\checkmark$ & 0.582 & 7.8 \\
\textbf{C-IQL (OPL-MT-MNAR)} & $\checkmark$ & \textbf{0.662} & \textbf{5.3} \\
\bottomrule
\end{tabular}
\end{adjustbox}

\caption{Constrained policy optimization on MIMIC-III. Richer MNAR-aware states remain beneficial even after adding explicit guideline constraints.}
\label{tab:constraints}
\end{table}

The constrained variant reduces guideline violations below the clinician baseline while preserving a substantial value advantage, supporting the view that better state estimation and explicit safety constraints are complementary rather than redundant.

\subsection{Text Attention Interpretability}

Table~\ref{tab:text_attention} lists representative high-attention phrases from the text branch and the clinical contexts they correspond to.

\begin{table*}[!tp]
\centering
\small
\begin{tabular}{@{}lcc@{}}
\toprule
\textbf{Keyword / Phrase} & \textbf{Attention} & \textbf{Clinical Context} \\
\midrule
worsening hemodynamic instability & 0.231 & Cardiovascular deterioration \\
new onset respiratory distress & 0.198 & Pulmonary decompensation \\
increased vasopressor requirements & 0.172 & Septic shock progression \\
positive blood cultures & 0.143 & Infection severity marker \\
\bottomrule
\end{tabular}
\caption{Representative high-attention nursing-note phrases (MIMIC-III).}
\label{tab:text_attention}
\end{table*}

A representative case is the note ``worsening hemodynamic instability,'' which receives high attention on deterioration tokens, a high gate value, and a corresponding increase in vasopressor intensity under the learned policy.

\subsection{Time-to-Deterioration Analysis}

We also evaluate the auxiliary head in a time-to-event setting by replacing the binary outcome head with a discrete-time hazard model over the same 72-hour ICU window. Table~\ref{tab:survival} summarizes the resulting short-horizon and longer-horizon deterioration metrics on MIMIC-III.

\begin{table}[H]
\centering
\normalsize
\setlength{\tabcolsep}{4pt}
\begin{adjustbox}{max width=\linewidth}
\begin{tabular}{@{}p{0.36\linewidth}cccc@{}}
\toprule
\textbf{Method} & \textbf{C-index} & \textbf{AOC@12h} & \textbf{AOC@24h} & \textbf{AOC@48h} \\
\midrule
Cox PH & 0.687 & 0.718 & 0.703 & 0.691 \\
Dynamic-DeepHit & 0.738 & 0.769 & 0.751 & 0.734 \\
OPL-MT-MNAR (Structured Only) & 0.769 & 0.801 & 0.781 & 0.758 \\
\textbf{OPL-MT-MNAR} & \textbf{0.821} & \textbf{0.859} & \textbf{0.832} & \textbf{0.801} \\
\bottomrule
\end{tabular}
\end{adjustbox}
\caption{Time-to-deterioration analysis on MIMIC-III. The MNAR-aware state is especially helpful for short-horizon deterioration prediction, where fresh documentation is most informative.}
\label{tab:survival}
\end{table}

The largest gain appears at 12 hours, which is consistent with the broader text-MNAR story: documentation-process signals are most useful for near-term changes in acuity.

\section{Computational Analysis}
\label{app:compute}

\paragraph{Model Size.}
Full model with text fusion: 1.29M parameters. Removing text fusion reduces to 1.05M parameters.

\paragraph{Training Time.}
On NVIDIA A100 (40GB): Stage 1 pre-training $\sim$4h, Stage 2 frozen RL $\sim$5h, Stage 3 fine-tuning $\sim$2.5h. Total: $\sim$11.5h per run. Removing text fusion reduces total training time by $\sim$19\% to 9.2h.

\paragraph{Inference.}
Batch inference (256 patients): 14.2ms latency, 8.4GB memory. This translates to sub-millisecond per-patient decisions, well within real-time clinical requirements.

\section{Extended Related Work}
\label{app:related}

% \subsection{Reinforcement Learning for Critical Care}

\subsection{Clinical Reinforcement Learning, Offline RL, and Off-Policy Evaluation}

Reinforcement learning for critical care has grown substantially since the AI Clinician \citep{komorowski2018artificial}, which established a common sepsis-RL setup based on 4-hour decision intervals, mortality-related rewards, and off-policy evaluation from observational ICU data. Follow-up work explored continuous state representations \citep{raghu2017continuous}, model-based approaches \citep{raghu2017modelbased}, and improved treatment policies under heterogeneous patient responses \citep{peng2018improving,tang2022model,huang2022continuous,sun2025exploring}. More broadly, offline RL has developed methods for stable policy learning under distribution shift, including Batch-Constrained Q-learning (BCQ) \citep{fujimoto2019off}, Conservative Q-Learning (CQL) \citep{kumar2020conservative}, AWAC \citep{nair2020awac}, and Implicit Q-Learning (IQL) \citep{kostrikov2022offline}. In parallel, off-policy evaluation (OPE) has provided tools such as importance sampling (IS) \citep{precup2000eligibility}, Per-decision importance sampling (PDIS) \citep{thomas2016data}, doubly robust estimators \citep{jiang2016doubly}, and bootstrap-based uncertainty quantification for assessing policies without deployment \citep{hanna2017bootstrapping}. In healthcare, these methods are especially important because online exploration is costly or infeasible. Our work adopts this offline RL and OPE perspective, but differs in learning patient states from partial multimodal observations whose missingness patterns are themselves informative.

\subsection{Missing Data in Clinical Time Series}

Missing data are pervasive in clinical time series and are often MNAR, since measurement and documentation depend on latent patient severity and clinician behavior \citep{little2019statistical,weiskopf2013methods,agniel2018biases}. In structured clinical time series, GRU-D \citep{che2018recurrent}, BRITS \citep{cao2018brits}, direct missingness modeling \citep{lipton2016directly}, and Raindrop \citep{zhang2022raindrop} account for irregular sampling and missingness, but largely focus on prediction rather than sequential decision-making. In multimodal EHR settings, methods such as M3Care \citep{zhang2022m3care}, MissModal \citep{lin2023missmodal}, DrFuse \citep{yao2024drfuse}, and MUSE \citep{wu2024muse} address missing modalities through fusion, disentanglement, or robustness mechanisms. Our setting differs in two ways. First, we focus on missingness that is endogenously driven by unobserved factors \citep{xiong2023,duan2024factor,duan2024target,li2024learning,li2024two,li2025generalizing,chen2026partial}. Second, rather than treating missingness only as a nuisance to accommodate, we use temporally evolving observation patterns across structured data and clinical text as signals for learning patient state. The most closely related work is \citet{liang2025causal}, which explicitly models informative missingness in multimodal EHR, but without temporal dynamics or policy learning.

\end{document}

%% file: diagram.tex
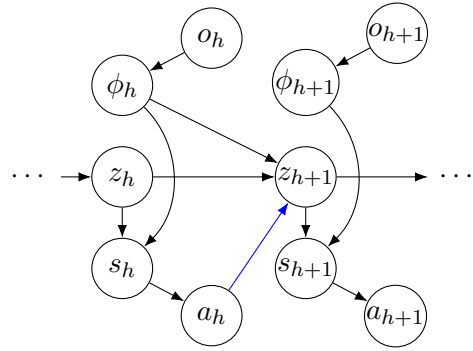
\begin{figure}[ht!]
    \centering
    % \vspace{-0.1cm} % tighten vertical spacing if needed
\begin{tikzpicture}[
    >=Latex,
    node distance=18mm,
    every node/.style={circle,draw,minimum size=8mm,inner sep=0pt},
    bend angle=18
]

\node (z) {$z_{h}$};
\node (phi) [above=4mm of z] {$\phi_h$};
\node (o) [above right=0.5mm and 6mm of phi] {$o_h$};
\node (s) [below=4mm of z] {$s_h$};
\node (a) [below right=0.5mm and 6mm of s] {$a_h$};

\node (z1) [right=16mm of z] {$z_{h+1}$};
\node (phi1) [above=4mm of z1] {$\phi_{h+1}$};
\node (o1) [above right=0.5mm and 6mm of phi1] {$o_{h+1}$};
\node (s1) [below=4mm of z1] {$s_{h+1}$};
\node (a1) [below right=0.5mm and 6mm of s1] {$a_{h+1}$};
\node[draw=none, shape=rectangle, inner sep=1pt] (dots1) [right=12mm of z1] {$\cdots$};

\node[draw=none] (past) [left=4mm of z] {$\cdots$};
\draw[->] (past) -- (z);

% Edges
\draw[->] (o) -- (phi);
\draw[->] (o1) -- (phi1);

\draw[->] (z) -- (z1);

\draw[->] (phi) to[bend left=45] (s);
% \draw[->, bend right=30] (phi) -- (s);
% \draw[->] (phi1) -- (s1);

\draw[->] (phi) -- (z1);

\draw[->] (z) -- (s);
\draw[->] (z1) -- (s1);

\draw[->] (s) -- (a);
\draw[->] (s1) -- (a1);

\draw[->, blue] (a) -- (z1);
\draw[->] (phi1) to[bend left=45] (s1);

\draw[->] (z1) -- (dots1);

% \draw[->] (h) -- (x);          % h -> X
% \draw[->] (h) -- (y);          % h -> Y
% \draw[->] (h) -- (d);          % delta -> h
% \draw[->] (x) -- (d);
% \draw[->, bend right=18] (d) to (y); 

\end{tikzpicture}
    % \vspace{-0.75cm}
    \caption{Causal diagram for patient state learning. 
    % Solid arrows denote deterministic encoding along the recorded observation path $o_h \to \phi_h \to s_h$, while dashed arrows denote stochastic, action-conditioned latent transitions through $a_h \to z_h$. This highlights why MNAR-aware observation encoding and action-conditioned dynamics are complementary: the former enriches the endogenous observation channel, and the latter preserves multi-step credit assignment (Theorem~\ref{thm:controllability}).
    }
    \label{fig:causal}
    % \vspace{-0.6cm}
\end{figure}